\newif\ifjournal  
\newif\ifblind    
\newif\ifanalysis

\ifjournal
  \ifblind
    \documentclass[sigconf,anonymous,nonacm]{acmart}
  \else
    \documentclass[sigconf,nonacm]{acmart}
  \fi

  \newcommand\vldbdoi{XX.XX/XXX.XX}
  \newcommand\vldbpages{XXX-XXX}
  \newcommand\vldbvolume{14}
  \newcommand\vldbissue{1}
  \newcommand\vldbyear{2020}
  \newcommand\vldbauthors{\authors}
  \newcommand\vldbtitle{\shorttitle}
  \newcommand\vldbavailabilityurl{URL_TO_YOUR_ARTIFACTS}
  \newcommand\vldbpagestyle{plain}
\else
  \documentclass[twoside,leqno,twocolumn]{article}
  \usepackage[margin=.7in]{geometry}

  \usepackage{xcolor}
  \usepackage{tikz}
  \usepackage{orcidlink}

  \usepackage{breakcites}
  \usepackage[square,numbers]{natbib}
  \renewcommand\cite{\citep}
  \setcitestyle{aysep={}}
\fi
\usepackage{tabularx}

\usepackage{algorithm}
\usepackage{algpseudocode}
\usepackage{amsmath}
\usepackage{graphicx}
\usepackage{textcomp}
\usepackage{url}

\usepackage{graphicx} 

\usetikzlibrary{graphs,quotes}
\usetikzlibrary{positioning}
\usepackage{hdbscan}

\makeatletter
\def\blfootnote{\xdef\@thefnmark{}\@footnotetext}
\makeatother

\definecolor{RedOrange}{cmyk}{0, 0.77,0.87,0}

\def\BibTeX{{\rm B\kern-.05em{\sc i\kern-.025em b}\kern-.08em
    T\kern-.1667em\lower.7ex\hbox{E}\kern-.125emX}}
    \hypersetup{draft}
\begin{document}

\title{PANDORA: A Parallel Dendrogram Construction Algorithm for Single Linkage Clustering on GPU}

\ifjournal
  \unless\ifblind
    \author{Piyush Sao}
    \orcid{0000-0002-9432-5855}
    \affiliation{%
      \institution{Oak Ridge National Laboratory}
      \streetaddress{1 Bethel Valley Rd}
      \city{Oak Ridge}
      \state{Tennessee}
      \country{USA}
      \postcode{37830}
    }
    \email{saopk@ornl.gov}

    \author{Andrey Prokopenko}
    \orcid{0000-0003-3616-5504}
    \affiliation{%
      \institution{Oak Ridge National Laboratory}
      \streetaddress{1 Bethel Valley Rd}
      \city{Oak Ridge}
      \state{Tennessee}
      \country{USA}
      \postcode{37830}
    }
    \email{prokopenkoav@ornl.gov}

    \author{Damien Lebrun-Grandi\'e}
    \orcid{0000-0003-1952-7219}
    \affiliation{%
      \institution{Oak Ridge National Laboratory}
      \streetaddress{1 Bethel Valley Rd}
      \city{Oak Ridge}
      \state{Tennessee}
      \country{USA}
      \postcode{37830}
    }

    \renewcommand{\shortauthors}{Sao and Prokopenko, et al.}
  \fi
\else
  \newcommand{\email}[1]{\href{mailto:#1}{\texttt{#1}}}
  \author{
    Piyush Sao\thanks{Oak Ridge National Laboratory}
    \\ \email{saopk@ornl.gov}
    \and 
    Andrey Prokopenko\footnotemark[1]
    \\ \email{prokopenkoav@ornl.gov}
    \and 
    Damien Lebrun-Grandi\'e\footnotemark[1]
    \\ \email{lebrungrandid@ornl.gov}
  }
  \date{}
\fi

\ifjournal
  \begin{abstract}
    This paper presents \pandora, a novel parallel algorithm for efficiently
constructing dendrograms for single-linkage hierarchical clustering, including
\hdbscan. Traditional dendrogram construction methods from a minimum spanning
tree (MST), such as agglomerative or divisive techniques, often fail to
efficiently parallelize, especially with skewed dendrograms common in real-world
data.

\pandora addresses these challenges through a unique recursive tree contraction
method, which simplifies the tree for initial dendrogram construction and then
progressively reconstructs the complete dendrogram. This process makes \pandora
asymptotically work-optimal, independent of dendrogram skewness. All steps in
\pandora are fully parallel and suitable for massively-threaded accelerators
such as GPUs.

Our implementation is written in Kokkos, providing support for both CPUs and
multi-vendor GPUs (e.g., Nvidia, AMD). The multithreaded version of \pandora is
2.2$\times$ faster than the current best-multithreaded implementation, while the
GPU \pandora implementation achieved 6-20$\times$ on \amdgpu and 10-37$\times$
on \nvidiagpu speed-up over multithreaded \pandora. These advancements lead to
up to a 6-fold speedup for \hdbscan on GPUs over the current best, which only
offload MST construction to GPUs and perform multithreaded dendrogram
construction.

  \end{abstract}
\fi

\maketitle

\ifjournal
  \pagestyle{\vldbpagestyle}
  \begingroup\small\noindent\raggedright\textbf{PVLDB Reference Format:}\\
  \vldbauthors. \vldbtitle. PVLDB, \vldbvolume(\vldbissue): \vldbpages, \vldbyear.\\
  \href{https://doi.org/\vldbdoi}{doi:\vldbdoi}
  \endgroup
  \begingroup
  \renewcommand\thefootnote{}\footnote{\noindent
  This work is licensed under the Creative Commons BY-NC-ND 4.0 International License. Visit \url{https://creativecommons.org/licenses/by-nc-nd/4.0/} to view a copy of this license. For any use beyond those covered by this license, obtain permission by emailing \href{mailto:info@vldb.org}{info@vldb.org}. Copyright is held by the owner/author(s). Publication rights licensed to the VLDB Endowment. \\
  \raggedright Proceedings of the VLDB Endowment, Vol. \vldbvolume, No. \vldbissue\ %
  ISSN 2150-8097. \\
  \href{https://doi.org/\vldbdoi}{doi:\vldbdoi} \\
  }\addtocounter{footnote}{-1}\endgroup

  \ifdefempty{\vldbavailabilityurl}{}{
  \vspace{.3cm}
  \begingroup\small\noindent\raggedright\textbf{PVLDB Artifact Availability:}\\
  The source code, data, and/or other artifacts have been made available at \url{https://github.com/anonymynona/2023-dendrogram}.
  \endgroup
  }
\fi

\unless\ifjournal
\begin{abstract}
  
\end{abstract}


\fi

\unless\ifblind
  \blfootnote {%
  This manuscript has been authored by UT-Battelle, LLC, under contract
  DE-AC05-00OR22725 with the U.S. Department of Energy. The United States
  Government retains and the publisher, by accepting the article for publication,
  acknowledges that the United States Government retains a nonexclusive, paid-up,
  irrevocable, world-wide license to publish or reproduce the published form of
  this manuscript, or allow others to do so, for United States Government
  purposes.
  }
\fi

\begin{figure}
  \centering
    \includegraphics[width=0.85\columnwidth]{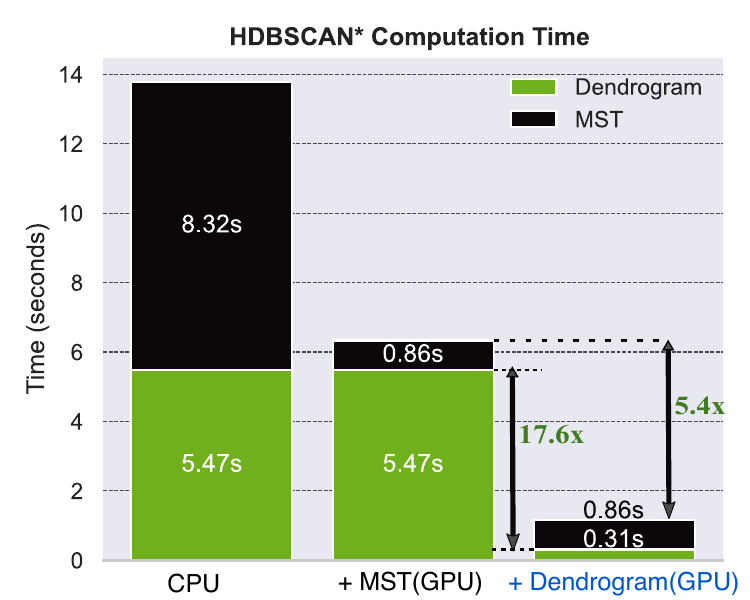}
    \caption{\label{f:overview}Time taken by \hdbscan components
    construction (Euclidean minimum spanning tree (MST) and dendrogram) on
    \amdcpu CPU and \amdgpu GPU for \emph{Hacc37M} dataset.}
\end{figure}


\section{Introduction}\label{sec:intro}



A dendrogram, a tree-like structure that encapsulates the hierarchical nature of
data, is a critical tool in machine learning. Its ability to visually represent
clusters' formation, merging, and splitting has applications in diverse fields,
from linguistics~\cite{divjak2014linguistic}, astronomy~\cite{feigelson2012astronomy} and
psychometry~\cite{yim2015psychology} to document classification~\cite{li1998classification}
and spatial-social network visualization~\cite{luo2011social}. In the
computational and molecular biology world, dendrograms are used for representing
phylogenetic trees~\cite{letunic2007phylotree}, elucidating gene
clustering~\cite{hodge2000myosin} and the evolutionary relationships among
biological taxa~\cite{farris1972estimating}. They are instrumental in decoding
gene co-expression \cite{xu2002gene,ovens2021comparative}, protein-protein interaction
networks, speciation rates, and genetic mutations
\cite{pfeifer2020molecular,zuckerkandl1965evolutionary}.

In most hierarchical clustering algorithms, the decision of how to combine or
split clusters is done through a use of a distance metric (e.g., Euclidean)
\cite{ward1963hierarchical,jarman2020hierarchical,gagolewski2016genie}.
Algorithms differ in their definition of the dissimilarity between two clusters.
In this work, we will focus on the \emph{single-linkage clustering}, which
defines the distance between two clusters to be the minimum distance
among all pairs of points such that the points in a pair do not belong to the
same cluster. Our choice is motivated by its use in the popular Hierarchical
Density-Based Spatial Clustering of Applications with Noise (\hdbscan)
algorithm~\cite{campello2015}.

%

In general, dendrogram construction is considered to be an inexpensive
operation. It is often done as step in a larger procedure. For example, in
\hdbscan, constructing MST (as part of \hdbscan) for high-dimensional data
relies on high-dimensional nearest neighbor search, an expensive procedure
dwarfing the dendrogram construction cost. Another reason is that the
dendrogram construction is often done for the datasets of moderate size.

None of these assumptions hold for the problems we explore in this work. We are
focus on the large datasets of the low dimensional data. In this case,
dendrogram construction becomes a dominant cost. \Cref{f:overview} (middle)
shows the status quo for an astronomy dataset \emph{Hacc37M}, where the MST
construction is performed on a GPU, and the dendrogram construction is done on
CPU. We see that the dendrogram construction takes 86\% of the overall time,
hampering the overall performance of the \hdbscan algorithm. The goal of this
paper is thus to bridge this gap by presenting a new parallel dendrogram
construction algorithm suitable for GPU architectures.

\begin{figure}[t]
    \includegraphics[width=\columnwidth]{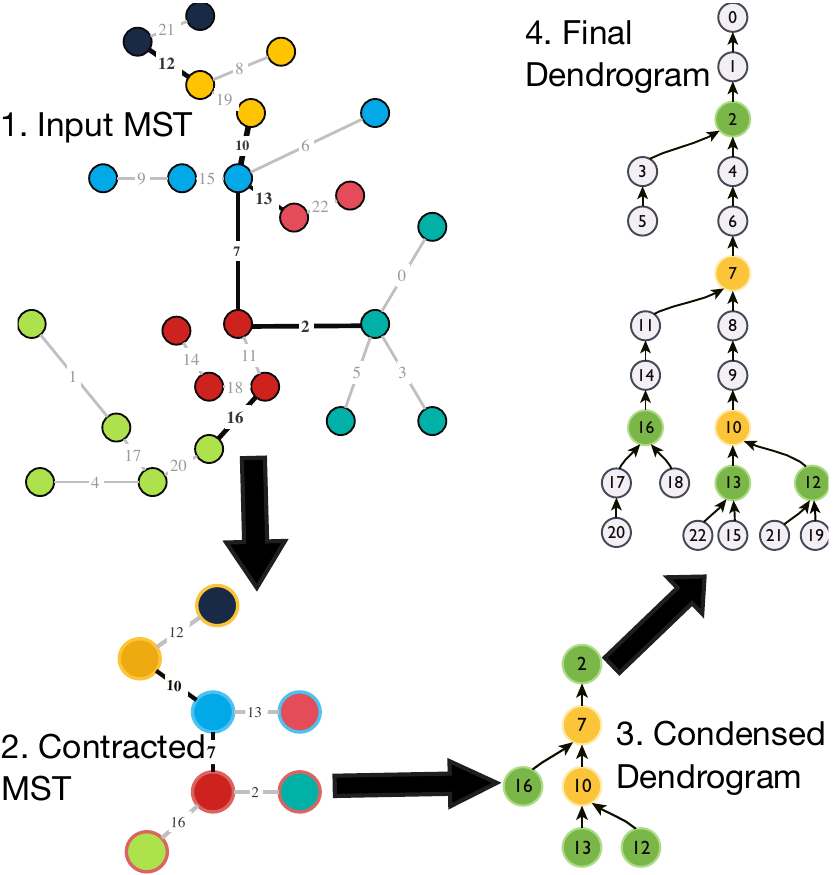}
    \caption{\label{fig:approach}
    A high-level visualization of the \pandora algorithm. The original MST (top
    left) is contracted (bottom left). Using the dendrogram corresponding to
    the contraction (bottom right), the original dendrogram is recovered by
    edge reinsertion from the same level. The dendrograms are shown using the
    edges numbering of the original MST; dendrogram leaf nodes, corresponding
    to the data points, are omitted.
    }
\end{figure}

\begin{figure}[t]
  \centering
    \includegraphics[width=0.64\columnwidth]{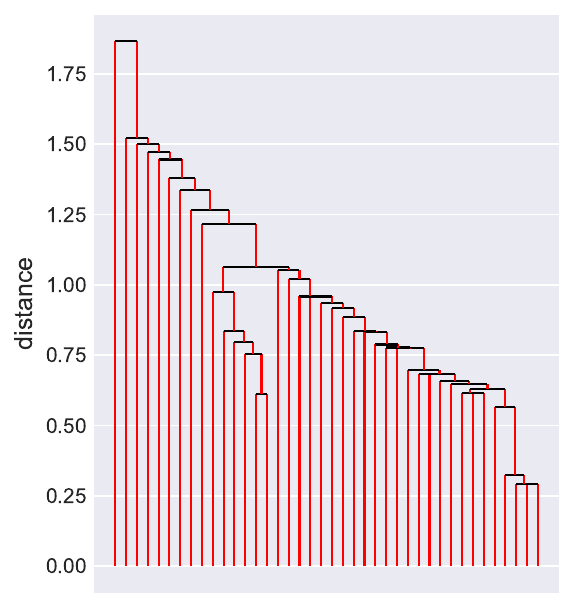}
    \caption{\label{f:dendrogram_example}An example of a highly skewed dendrogram
    constructed from a 40 point sample taken from a 3D Gaussian distribution
    using \hdbscan mutual reachability distance with $minPts = 2$.}
\end{figure}


We introduce a novel parallel algorithm for constructing single-linkage
dendrograms. Our approach is very different from the existing top-down and
bottom-up approaches (see~\Cref{sec:background}). It effectively handles highly
skewed dendrograms such as shown on~\Cref{f:dendrogram_example}, which are
common as we will demonstrate. The algorithm
contracts a subset of edges in an MST to yield its coarser version. We
efficiently compute the dendrogram for this coarse MST and reconstruct the
dendrogram for the original MST by reintegrating all the previously contracted
edges. The contraction is applied recursively to ensure the work optimality of
our algorithm.
This process is illustrated in \Cref{fig:approach} using an MST
with two levels of contraction. \Cref{f:overview} (right) shows that we are
able to achieve our goal, improving the construction time by $17\times$, now
taking 26\% of the overall time for this dataset.


Our proposed parallel dendrogram construction algorithm is work-optimal, highly
parallel, and efficient, even when dealing with highly skewed dendrograms. All
the steps are highly parallelizable and can be easily adapted for multicore
CPUs and GPU architectures using parallel constructs such as parallel loops,
reductions and prefix sums. We implemented our algorithm using the
performance-portable Kokkos library~\cite{kokkos2022}, which marks the first
known GPU implementation for dendrogram construction.

We evaluate our algorithm on various real-world and artificial datasets, and
compare it to the best-known open-source multi-core CPU implementation. Our
experimental results reveal that our multi-core CPU implementation is twice as
fast as the state-of-the-art, while the GPU variant achieves 15-40$\times$
speedup compared to the multi-core CPU performance.

This paper presents several algorithmic, theoretical, and practical
contributions to dendrogram construction in a single linkage clustering
algorithm. These contributions are as follows:
\begin{itemize}
\item We propose a novel tree contraction-based approach for constructing
dendrograms in a highly parallel manner (\Cref{sec:alpha-tree}).
\item We analyze the edge contraction technique that can be used with our
algorithm and establish the necessary and sufficient conditions for its use
(\Cref{sec:correctness}). 
\item We derive the asymptotic lower-bound complexity for any dendrogram
construction algorithm (\Cref{thm:lowerbound}) and prove that our algorithm is
work-optimal(\Cref{sec:complexity}).
\item We provide an efficient and performance-portable algorithm implementation
for our algorithm ported to three different architectures (\Cref{sec:implementation}). In
particular, the paper presents the first GPU implementation for dendrogram
construction for~\hdbscan, which is faster by a factor of 15-37$\times$ over
multithreaded implementation (\Cref{sec:results}).
\end{itemize}

Our work significantly advances the state-of-the-art in parallel \hdbscan
clustering computation on GPUs for large datasets. Combined with recent
developments in parallel Euclidean MST computation on
GPUs~\cite{prokopenko2022}, our algorithm facilitates rapid \hdbscan clustering
computation on modern hardware for low dimensional datasets. For example, a
single Nvidia A100 GPU can now compute \hdbscan clustering in under one second
for a 37M cosmological problem, and around six seconds for a 300M uniformly
distributed point cloud.

\section{Background}\label{sec:background}
\begin{figure*}[thbp]
  \includegraphics[width=\textwidth]{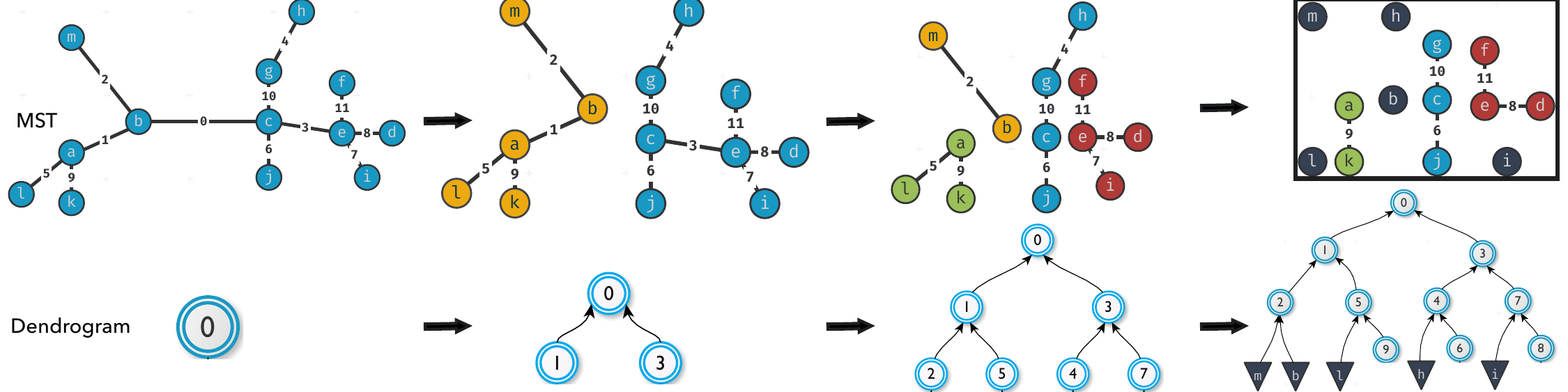}
  \caption{\label{fig:topdown} Steps in top-down dendrogram
  construction(\Cref{sec:divisive}). The dendrogram starts with the heaviest
  edge in \mst as the root, which is removed to divide the tree into two
  connected components. In subsequent steps, the heaviest edge among all components
  is identified, and their parent is the heaviest edge from the previous
  step. This process is repeated recursively for each component of the tree.}
\end{figure*}

\subsection{Single linkage clustering}
\label{sec:singlelinkage}
Single linkage clustering~\cite{gower1969singlelink}, also known as nearest
neighbor clustering, is a hierarchical clustering method that measures the
distance between two clusters by the shortest distance between any pair of
points in the clusters. This approach is particularly adept at identifying
clusters with irregular shapes and non-elliptical forms, setting it apart from
complete linkage and average linkage clustering, which use the longest and
average pairwise distances, respectively.

Our primary interest lies in \hdbscan, a widely-used density-based clustering
algorithm and a variant of single-linkage clustering. Single linkage clustering
is versatile, finding applications across various fields such as bioinformatics,
astronomy, and image processing. In bioinformatics, SLINK is employed for
analyzing gene expression~\cite{sibson1973slink}. Astronomers use the
Friends-of-friends algorithm, which is based on single linkage principles, to
detect galaxy clusters~\cite{feigelson2012astronomy}. In the realm of image
processing, it facilitates tasks like image segmentation by grouping similar
pixels~\cite{xu1996segmentation} and explores the morphological characteristics
of images using tree structures \cite{bosilj2018survey, soille2008, havel2013}.
Furthermore, clustering methods that rely on the Minimum Spanning Tree (MST) are
also derived from single linkage clustering \cite{zahn1971graph, xu2002gene,
laszlo2005minimum, jana2009efficient}.

Single linkage clustering methods typically start by generating a minimum
spanning tree (MST). This process requires a distance graph or matrix storing
the pairwise distances between points or nodes. To compute the MST, one can
employ algorithms like Prim's \cite{prim1957}, Kruskal's \cite{kruskal1956},
or~{\boruvka}'s \cite{boruvka1926}.

If clustering graphs or networks, the distance graph is already provided in some
cases. However, in the case of clustering spatial points using \hdbscan, the
graph is not explicitly formed. Instead, spatial search trees such as 
kd-trees \cite{bentley1978,wang2021,prokopenko2022} 
are utilized.

The second step in single linkage clustering is to construct the dendrogram to
extract a hierarchical cluster structure. Once the minimum spanning tree is
built, a dendrogram is produced using the MST, which we will explain in greater
detail.

\subsection{Dendrogram}
\textbf{Dendrogram:} A dendrogram is a directed tree structure used to represent
hierarchical clustering. The tree consists of leaf nodes that represent data
points and internal nodes that represent clusters of data points. The
relationships between clusters are conveyed through directed edges, indicating
whether a cluster contains or is contained in another cluster.

\textbf{Dendrogram in
Single-Linkage Clustering:}  Single-linkage clustering methods use the edges of
the MST to represent the internal nodes of the dendrogram. If an edge is removed
from the MST, it indicates splitting a cluster into two smaller clusters. As a
result, clusters correspond to MST edges, and removing them leads to cluster
separation. Removing an edge can only divide a cluster into two, hence the
dendrogram is typically a binary tree.




\subsection{Dendrogram construction algorithms}
\subsubsection{\textbf{Top-down dendrogram construction}}\label{sec:divisive}

%
\begin{figure}[t]
\captionsetup{labelformat=empty}
\input{algs/dgram-td.tex}
\end{figure}
%
The top-down approach for constructing a dendrogram from a given Minimum
Spanning Tree (MST) employs a divide-and-conquer strategy. This method hinges on
the principle that the dendrogram's root corresponds to the largest edge in the
MST. Removing this edge divides the tree into two subtrees, which may be as
small as a single vertex. These subtrees form the child nodes of the removed
edge in the complete dendrogram. The process is then recursively applied to each
subtree, removing the largest edge until all edges are eliminated. First three
steps for an example \mst is shown in~\Cref{fig:topdown}. Pseudo-code for this
method is presented in \Cref{alg:dgram-td}.

The top-down approach has several drawbacks. Primarily, the top-down approach
often underperforms with skewed dendrograms, which is common in real-world data.
For optimal performance, the algorithm must split the tree by removing the
largest edge, aiming for two subtrees of equal size. However, in skewed
scenarios, this can result in highly disproportionate subtree sizes,
occasionally reducing one subtree to a single vertex. This results in:
\begin{itemize}
  \item \textit{Increased asymptotic cost:} The cost of the algorithm is
  $O(nh)$, with $h$ representing the dendrogram's height. In the case of skewed
  dendrograms, this cost surpasses the $O(n\log n)$ cost associated with
  well-balanced dendrograms. Therefore, the top-down algorithm is not
  work-optimal for highly skewed dendrograms.
  \item \textit{Limited parallelism:}
    The imbalance of two subtrees after an edge removal limits the available
    parallelism. Moreover, the computational depth (the number of required
    parallel steps) is $O(h)$, much higher than the ideal $O(\log n)$.
\end{itemize}

Another limitation of this approach is that its efficiency depends on the
availability of an efficient algorithm for the MST (Minimum Spanning Tree) split
operation. This operation is efficient when the MST is in Euler tour form, but
this is not always the case, such as when the MST is constructed using
Bor\r{u}vka's algorithm~\cite{boruvka1926}. Although parallel algorithms exist
to convert a tree from a set of edges to an Euler tour~\cite{polak2021}, they
are often costly in practice. These limitations make the top-down approach less
effective and efficient when dealing with real-world datasets.

\subsubsection{Bottom-up dendrogram construction}\label{sec:agglomerative}
\begin{figure}[t]
\captionsetup{labelformat=empty}
\input{algs/dgram-uf.tex}
\end{figure}
The bottom-up processes the edges in order from the smallest to the largest.
For each edge, it identifies the clusters containing its vertices, and creates
a new parent cluster by merging the vertices' clusters. To keep track of the
cluster membership, it utilizes the union-find structure~\cite{tarjan1984worst}.
\Cref{alg:dgram-uf} shows pseudo-code for this approach.

In contrast to the top-down algorithm, the bottom-up approach is work-optimal
for any dataset. The most demanding operation, sorting, has $O(n\log n)$
complexity. Processing edges has an asymptotic cost $O(n \mathcal{A}(n))$,
where $\mathcal{A}(n)$ represents the inverse Ackerman function \cite{tarjan1984worst}.
This gives the overall worst-case time complexity of $O(n\log n)$.

The main drawback of the algorithm is that the edges can only be processed
sequentially. For a given edge, it is impossible to say when it should be
processed given the information only about its vertices or adjacent edges. This
is due to the non-local nature of the dendrogram, where the parents of an edge
may come from a completely different part of the graph. Thus, standard methods to parallelize
sequential algorithms, such as~\cite{blelloch2012}, cannot be used here.

\subsubsection{Mixed dendrogram construction}
\label{sec:mixed}
Wang et al. \cite{wang2021} combined top-down and bottom-up approaches to create
a parallel algorithm for the shared memory architecture. The algorithm avoids
the limitations of the sequential bottom-up approach by first removing a set of
the largest edges (a tenth or a half) in a top-down fashion, splitting the tree into
several subtrees. The dendrograms for the subtrees and the top tree are
constructed using the bottom-up approach, then stitched together.

The algorithm exhibits higher degree of parallelism compared to the sequential
counterpart. However, it is still subject to the same limitations for
constructing highly skewed dendrograms, leading to work inefficiency and
imbalance, particularly problematic on GPUs. Moreover, it relies on the Euler
tour implementation for partitioning. Euler tour construction heavily depends
on parallel list-ranking, which significantly underperforms on GPUs compared
to prefix-sum or sort algorithms.

\subsubsection{Image morphological trees}
Single linkage clustering variants commonly used in image analysis include the
max tree, component tree, omega tree, binary partition tree, and
min-tree \cite{bosilj2018survey,soille2008,havel2013}. These trees are similar
to dendrograms used in hierarchical clustering, but they are constructed from
images instead of data points. Some attempts at parallelization have been made
using multithreading \cite{havel2019alphatree} and GPU \cite{blin2022max}.
Although the computation of these trees and dendrograms have similarities, the
requirements differ significantly due to the tree's structure. Their method for
parallelization involves partitioning the image, which results in a complexity
of $O(nh)$, which becomes $O(n^2)$  in the worst case when height is $O(n)$. We
believe our algorithm can be modified to work for these problems.

\section{ PANDORA: parallel dendrogram computation using tree contraction}
\label{sec:alpha-tree}
\newcommand{\mycomment}[1]{\State{{\textcolor{blue}{\(\triangleright\) \small \textit{#1}}}}}

\begin{algorithm}[t]
\caption{Dendrogram Computation using Tree Contraction}
\label{algm:pandora}
\begin{algorithmic}[1]
\Require $T = (V, E)$: minimum spanning tree
\State $E_{\alpha} \gets$ Find edges of contracted tree
\State $T_{\alpha} \gets$ Construct contracted tree by contracting edges in $E - E_{\alpha}$ in $T$
\State $P_{\alpha} \gets$ Compute dendrogram of contracted tree $T_{\alpha}$
\mycomment{Construct complete dendrogram $P$ from $P_{\alpha}$}
\For{ \textbf{each edge $e$ in $E - E_{\alpha}$ \textbf{in parallel}}}
    \mycomment{Find the chain of $e$ using contracted dendrogram $P_{\alpha}$}
    \State $P_{\alpha}(e) \gets$ Find parent of $e$ in $P_{\alpha}$
    \mycomment{Map $e$ to its corresponding chains in $P$}
    \State $C \gets$ Determine of chain containing  $e$
    \State Add $e$ to set of edges in the chain $C$
\EndFor
\mycomment{Order and connect chains to form $P$}
\For{each chain $C$}
    \State Sort edges in $C$ by their index in $E$
    \For{each edge $e$ in $C$, excluding the first}
        \State $P(e) \gets$ Find predecessor of $e$ in $C$
    \EndFor
    \If{$e_s$ is first edge in sorted chain}
    \State $P(e_s) \gets$ $\alpha$-edge for chain $C$
    \EndIf 
\EndFor
\State Connect chains to form the complete dendrogram $P$
\State \textbf{return} $P$
\end{algorithmic}
\end{algorithm}



\begin{figure}[t]
  \centering
  \includegraphics[width=\columnwidth]{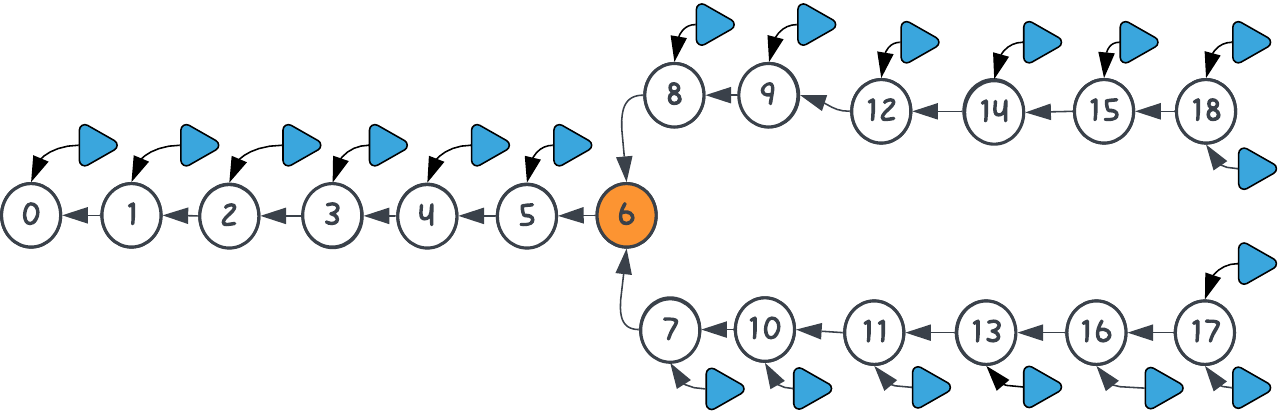}
  \caption{\label{fig:invertedY}\figcaptsize The \NewAlgName leverages
  dendrogram chains to construct them efficiently. This dendrogram can be
  divided into three chains: top, bottom-left, and bottom-right. } 
\end{figure}

\NewAlgName leverages dendrogram chains—continuous segments without branching,
present in highly skewed dendrograms.
Notably, within a chain, the edges are organized
by their index. Consider an inverted Y-shaped dendrogram (\cref{fig:invertedY}):
it consists of three chains—top, bottom-left, and bottom-right. By assigning
each edge to its respective chain, we can efficiently sort and link the chains
to reconstruct the full dendrogram.

To identify these chains, we employ a tree contraction method, which generates a
condensed version of the original dendrogram. It condenses each dendrogram chain
into a single edge in the contracted dendrogram, allowing us to map all edges to
a dendrogram chain and construct the complete dendrogram.

\NewAlgName operates in two main stages. The first is the recursive tree
contraction (\Cref{sec:rtreecontraction}), where we strategically reduce the
tree's size by contracting specific edges. We apply this contraction recursively
to determine the dendrogram of the reduced tree. The second stage is the
dendrogram expansion (\Cref{sec:expansion}), which involves piecing together the
full dendrogram, starting from the most contracted state and incrementally
expanding it.

Careful consideration is necessary during tree contraction to ensure only valid
contractions are performed. The algorithm's correctness, including the criteria
for valid contractions, is discussed in \Cref{sec:correctness}.

Furthermore, in \Cref{sec:complexity}, we will explore the minimum complexity
any dendrogram construction algorithm must adhere to and show that our algorithm
meets this optimal bound. 

We begin by defining essential terms and notations
for our algorithm's description.




%
\begin{figure*}[t]
\subfloat[\label{fig:yed:MST10} An \mst~(not to scale)]{\includegraphics[height=0.21\textheight]{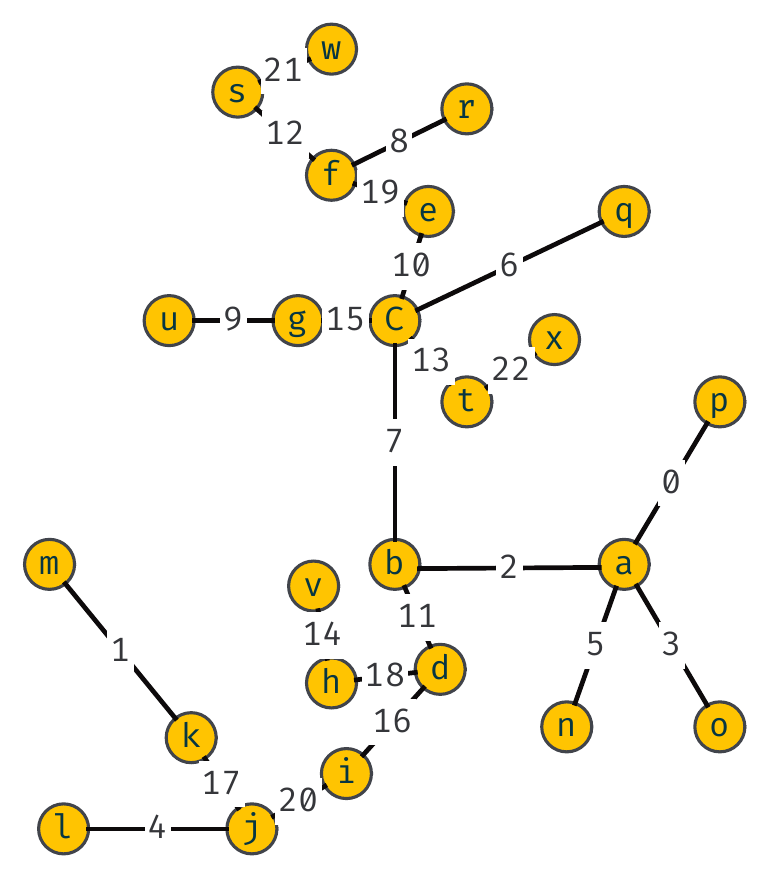}}
\hfill
\subfloat[\label{fig:yed:MST15} highligted $\alpha$-edges]{\includegraphics[height=0.21\textheight]{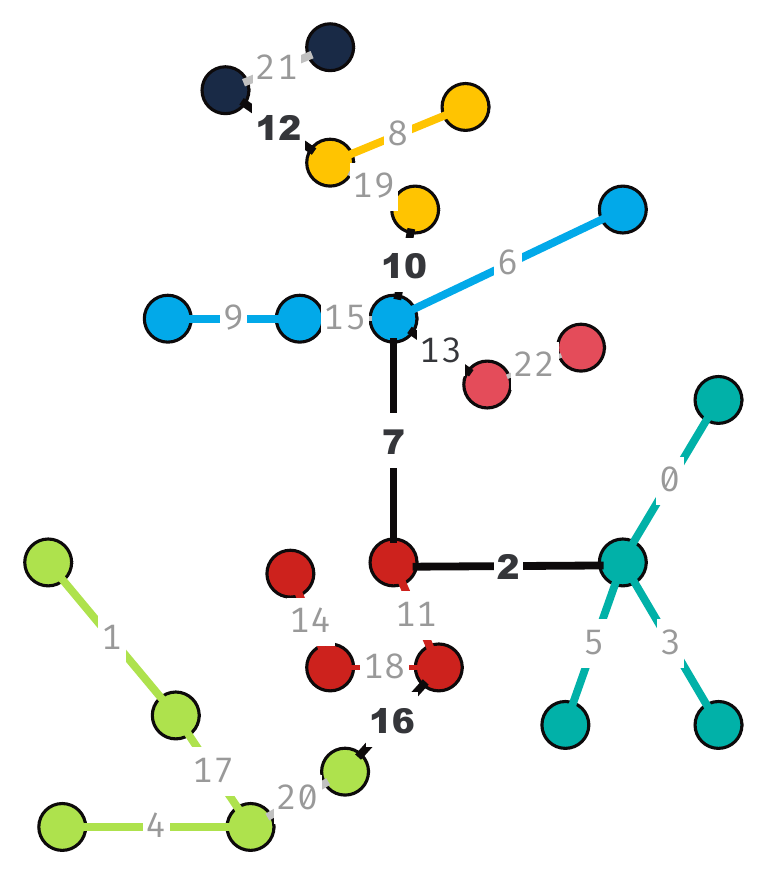}}
\hfill
\subfloat[\label{fig:yed:MST20} $\alpha$-\mst]{\includegraphics[height=0.21\textheight]{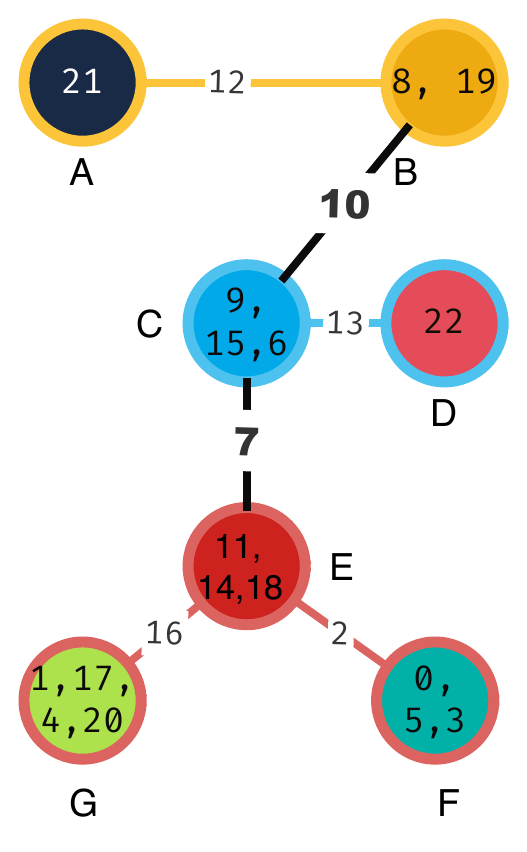}}
\hfill
\subfloat[\label{fig:yed:MST30}
$\beta$-\mst]{\includegraphics[height=0.22\textheight]{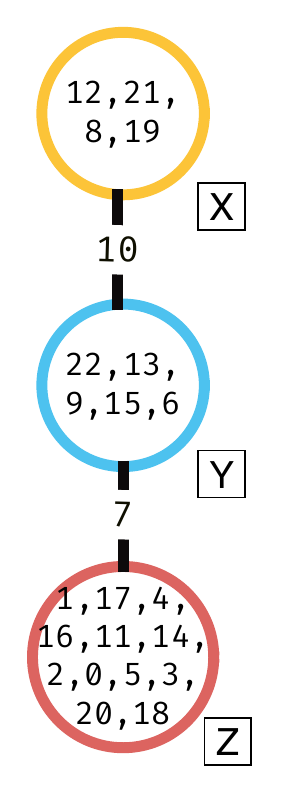}}
\caption{\label{fig:yed}\figcaptsize Recursive Tree Contraction
(\Cref{sec:rtreecontraction}). The original~\mst is shown in
\cref{fig:yed:MST10}. \cref{fig:yed:MST15} highlights the $\alpha$-edges of the
same~\mst. Removing these edges results in a division of the tree into
components, each marked with a different color in \cref{fig:yed:MST15}. These
components are then contracted into $\alpha$-vertices. The $\alpha$-vertices
and $\alpha$-edges form the $\alpha$-\mst, the first level of contraction,
depicted in \cref{fig:yed:MST20}. This contraction process continues to a second
level to form the $\beta$-\mst, as shown in \cref{fig:yed:MST30}. Both $\alpha$ and
$\beta$~\mst encapsulate their respective edges within supervertices. } 
\end{figure*}

\subsection{Terminology and notation}
\label{sec:notation}
\subsubsection{Minimum spanning tree structure}
Consider a Minimum Spanning Tree (MST) $T=\{ V, E, W\}$, where we aim to
calculate its dendrogram. Let $n_{v}$ denote the number of vertices, and
$n=n_{v} -1$ represent the number of edges in the MST. The dendrogram
computation  begins with sorting the edges in $T$ by weight in descending order,
which requires $O(n \log n)$ time. This sorting step is crucial for the
dendrogram's computation, ensuring that edges with equal weights are ordered
consistently to preserve the dendrogram's uniqueness and facilitate the
validation of our method. For subsequent discussions, we will assume that the
edges are already sorted in this manner.  We use the following notation to
describe the incidence structure of the tree.

\textbf{Incident edges:} For a vertex $v \in V$, the set
$\Incident{v}$ includes all edges incident to $v$. For example, as shown in
\Cref{fig:yed:MST10}, $\Incident{a}$ comprises the edges $ \{ e_{0}, e_{2},
e_{3}, e_{5} \}$.

\textbf{Maximum incident edge:} $\maxIncident{v}$ denotes the edge with the
highest index in $\Incident{v}$. From the previous example, $\maxIncident{a} =
e_{5}$.

\textbf{Neighboring edges $\edgeneighbor{e}$:} For any edge $e$, the set
$\edgeneighbor{e}$ consists of edges that share a vertex with $e$. Specifically,
if $e$ connects vertices $v$ and $u$, then $\edgeneighbor{e} = \Incident{v} \cup
\Incident{u}$.

\textbf{Edge contraction of a tree}
We can create a contracted tree $T_{c}= (V_c, E_c)$ from a tree $T$ and a subset
of edges $E_{c}$. To do this, we contract the edges in the set $E - E_{c}$.
Initially, $V_{c}$ is identical to $V$. For each edge $e = (u, v) \in E -
E_{c}$, we merge $u$ and $v$ into a single supervertex $vu$, removing $u$ and
$v$ from and adding supervertex $vu$ to $V_{c}$. The supervertex $vu$ inherits
the neighbors of $u$ and $v$, except for $u$ and $v$ themselves. This
contraction is repeated for all edges in $E - E_{c}$. The resulting contracted
tree $T_{c}$ comprises the modified vertex set $V_{c}$ and the edge subset
$E_{c}$.

\subsubsection{Dendrogram structure}
A dendrogram is a directed rooted binary tree, denoted as $\mathcal{D} = \{V_{d},
E_{d} \}$. Its vertex set $V_{d}$ comprises two types of nodes: vertex nodes,
representing the vertices of the Minimum Spanning Tree (MST), and edge nodes,
representing the MST's edges. Thus, we have $V_d = V \cup E$, with vertex nodes
located at the leaves corresponding to individual data points, and edge nodes as
internal nodes signifying clusters.

The dendrogram's structure is established through directed links that outline
parent-child relationships between nodes. These relationships determine the edge
set $E_d$, as defined by the parent function $P$. Specifically, $E_d$ is
composed of directed edges $(v \rightarrow u)$ where $P(v) = u$, with $v$ being
a member of $V_d$—either a vertex or an edge of the MST—and $u$ representing an
edge in the MST. Thus, dendrogram computation is equivalent to determining the
parent  $P$ for all nodes in $V_d$.


\paragraph{Parent of a vertex-node:} In a dendrogram, the parent of a vertex-node
$v\in V$ is the edge that disconnects $v$ from the tree when removed during the
top-down process. This process entails sequentially eliminating edges in
$\Incident{v}$, beginning with the heaviest (the smallest index) and concluding with
the lightest (the largest index). Thus, the dendrogram parent of vertex $v$ is the
edge in $\Incident{v}$ with the largest index.
\begin{equation}
  \label{eq:vertexparent}
  P(v) = \maxIncident{v} \qquad \forall v \in V;
\end{equation}
For example, in \Cref{fig:yed:MST10}, $P(a) = e_{5}$. The incidence structure of
the tree allows us to determine the parents of all vertex nodes $v \in V$.
However, \emph{identifying the parents of the edge nodes} presents the main
challenge.

\begin{figure}
  \centering
  \includegraphics[width=\columnwidth]{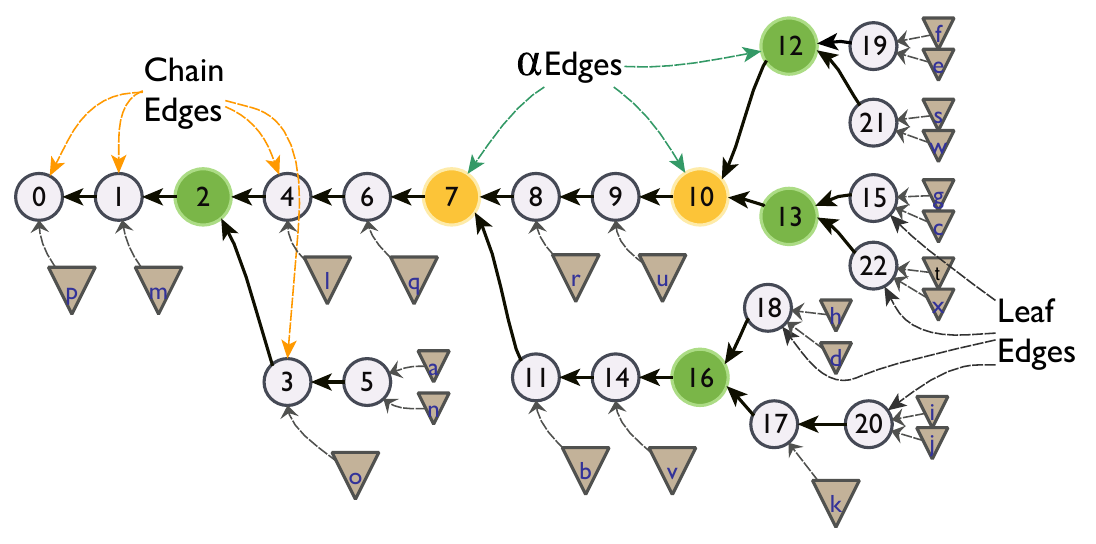}
  \caption{\label{fig:edgetaxonomy}\figcaptsize Dendrogram corresponding
  to the MST in \Cref{fig:yed:MST10}. We mark the vertex nodes as triangles and
  edge nodes as circles. The edge nodes are further classified into leaf, chain,
  and $\alpha$-edges.}
\end{figure}

\paragraph{Types of edge nodes:}
%
We can classify edge nodes in a dendrogram based on the number of vertex nodes
they have as their children. In a binary dendrogram, each edge node has exactly
two children, which can be either an edge or a vertex node. This leads to three
types of edge nodes (shown in \Cref{fig:edgetaxonomy}):
\begin{itemize}
  \item \textbf{Leaf edges:} have two vertex nodes as children. 
  \item \textbf{Chain edges:} have one vertex node and one edge node as children.
  \item \textbf{$\alpha$-Edges:}  do not have any vertex node as a child; both of their children are edge nodes as well.
\end{itemize}

Using the MST's local incidence structure and~\Cref{eq:vertexparent},
we can identify the parent of a vertex node. This information allows us to
calculate the number of children for any edge node. Consequently, we can
classify the edge node as a leaf, chain, or $\alpha$-edge based on its local
incidence structure alone. 

However, discerning the parent of an edge node
through this local structure alone is challenging due to the more complex
parent-child relationships between edge nodes, which often extend beyond
immediate neighbors in the MST. 
For example, in
\Cref{fig:yed:MST15}, edge node $e_2$ has $e_1$ as its parent, yet $e_1$ and
$e_2$ are situated in completely separate sections of the tree.

\subsubsection{Dendrogram chains and skewness}
\textbf{Dendrogram chains:}A dendrogram 'chain' is a lineage in a dendrogram
that extends without branching. It comprises a series of chain edges followed by
a final non-chain edge, which can be either a leaf or an $\alpha$ edge. Each
edge node in the chain, except for the last one, has a single child that is the
subsequent edge in the chain. The chain's end is marked by a leaf or an $\alpha$
edge. Chains ending in a leaf edge are  referred to as \emph{leaf chains}.

\textbf{Skewness of the Dendrogram:} We define a dendrogram's skewness as the
ratio of the height of the dendrogram to its ideal height = $\log_{2} n$. A
large number of chains in the dendrogram can lead to an increase in its height
and consequently, its skewness.

Developing a parallel dendrogram algorithm is difficult because real-world
dendrograms are often highly-skewed. Even dendrograms constructed from low
dimensional Gaussian distributions have heights far from a balanced tree height.
 This is a common occurrence, as we demonstrate in our
results section for various datasets, from GPS location data to cosmology and
power usage (see \Cref{tab:datasettable}.)

\begin{figure}[t]
  \centering
  \subfloat[\label{fig:yed:dgram20}$\alpha$-Dendrogram]{\includegraphics[height=0.18\textheight]{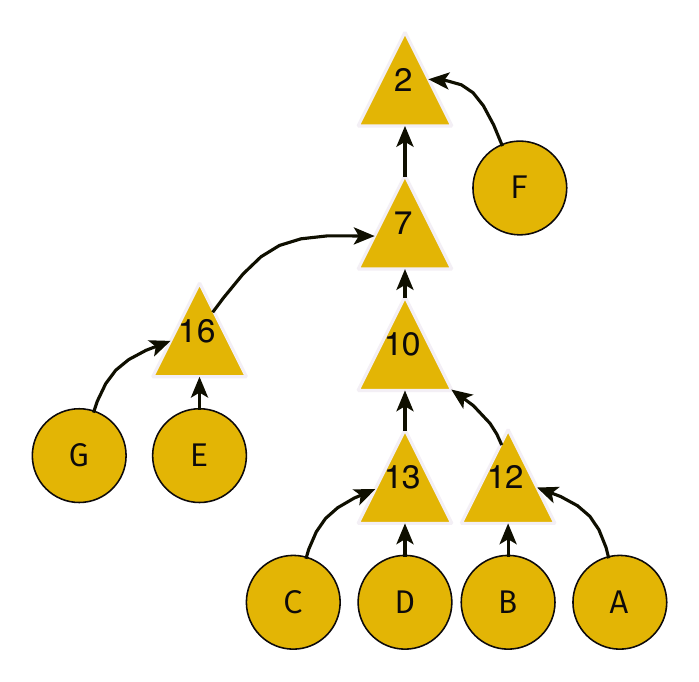}}
  \hfill
  \subfloat[\label{fig:yed:dgram30}$\beta$-Dendrogram]{\includegraphics[height=0.18\textheight]{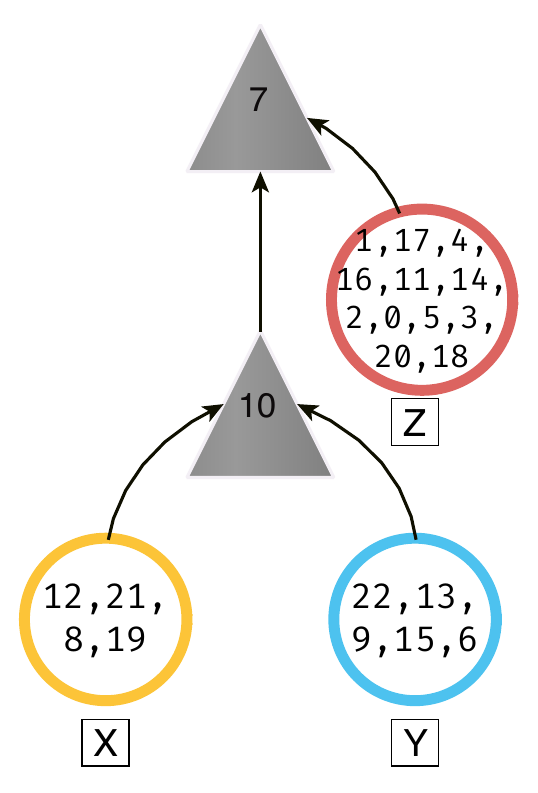}}
  \caption{\label{fig:yed:abdgrams}\figcaptsize The $\alpha$-dendrogram and
  $\beta$-dendrogram for the $\alpha$-\mst(\Cref{fig:yed:MST20}) and
  $\beta$-\mst (\Cref{fig:yed:MST30}) respectively.}
\end{figure}


\begin{figure*}
  \centering
  \subfloat[\label{fig:yed:exp0}$\alpha$-Dendrogram]{\includegraphics[width=0.19\textwidth]{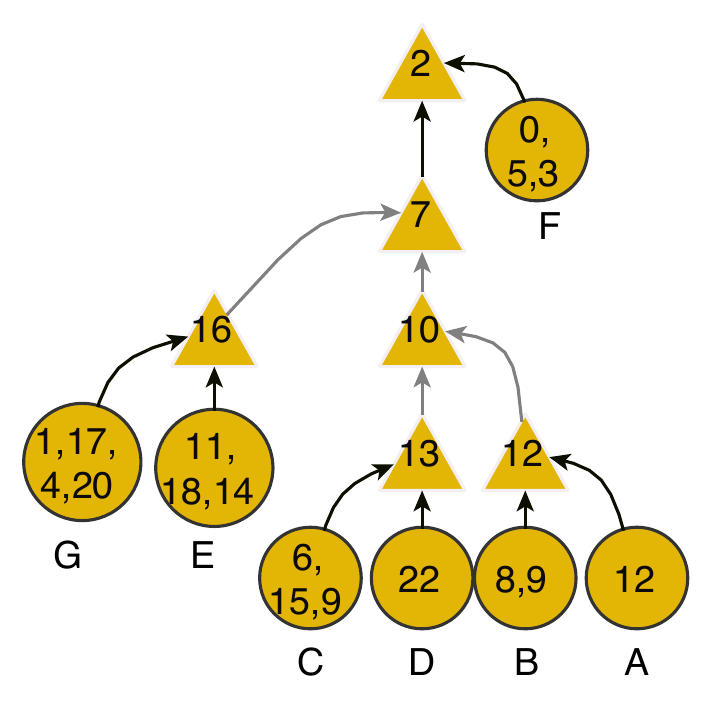}}
  \subfloat[\label{fig:yed:exp1}  $\alpha$-Leaf chains]{\includegraphics[width=0.2\textwidth]{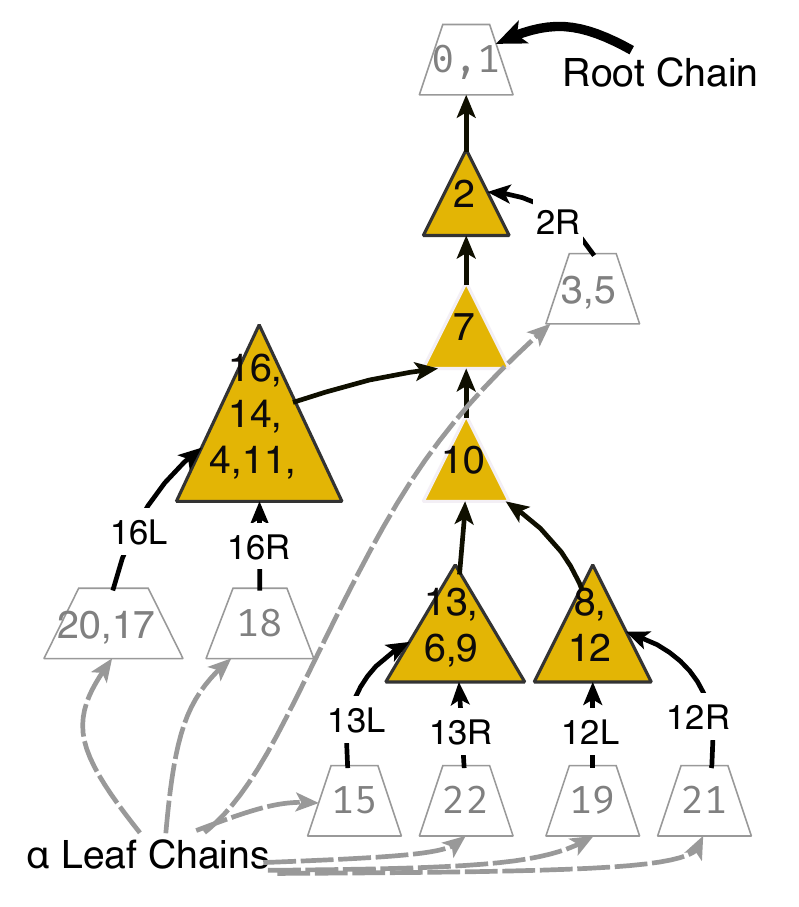}}
  \subfloat[\label{fig:yed:exp2} $\beta$-Leaf chains]{\includegraphics[width=0.2\textwidth]{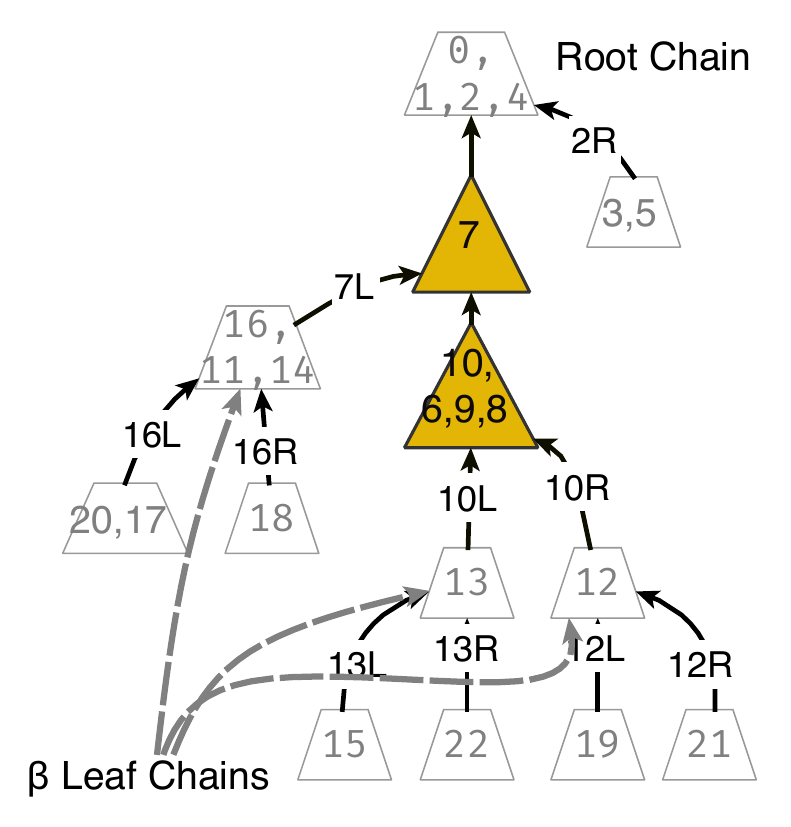}}
  \subfloat[\label{fig:yed:exp3}\footnotesize Final chains]{\includegraphics[width=0.21\textwidth]{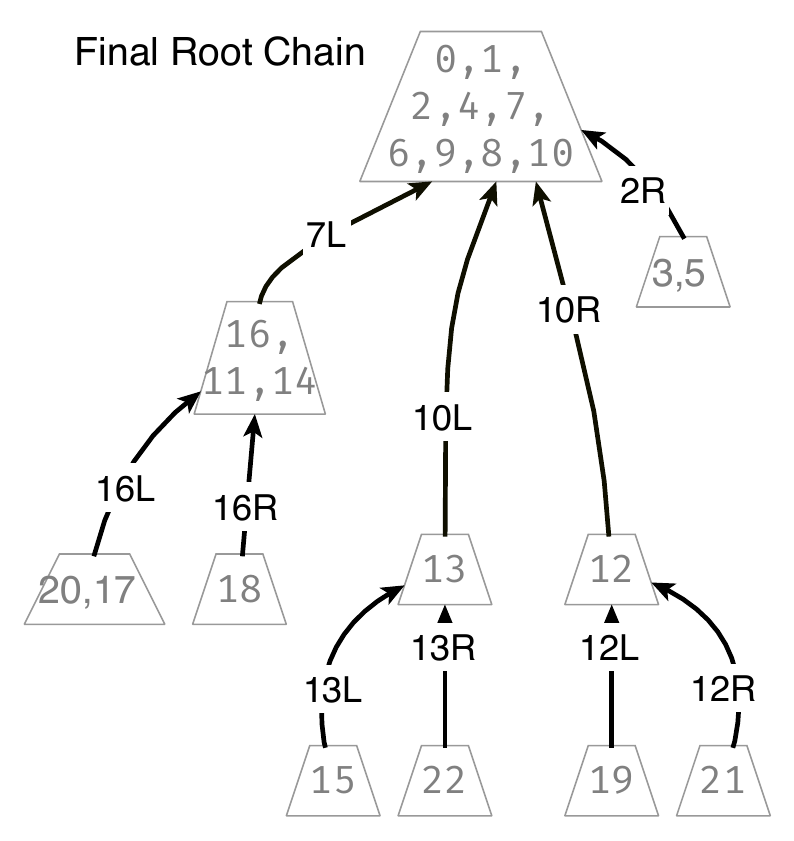}}
  \subfloat[\label{fig:yed:exp4}Final Dendrogram]{\includegraphics[width=0.2\textwidth]{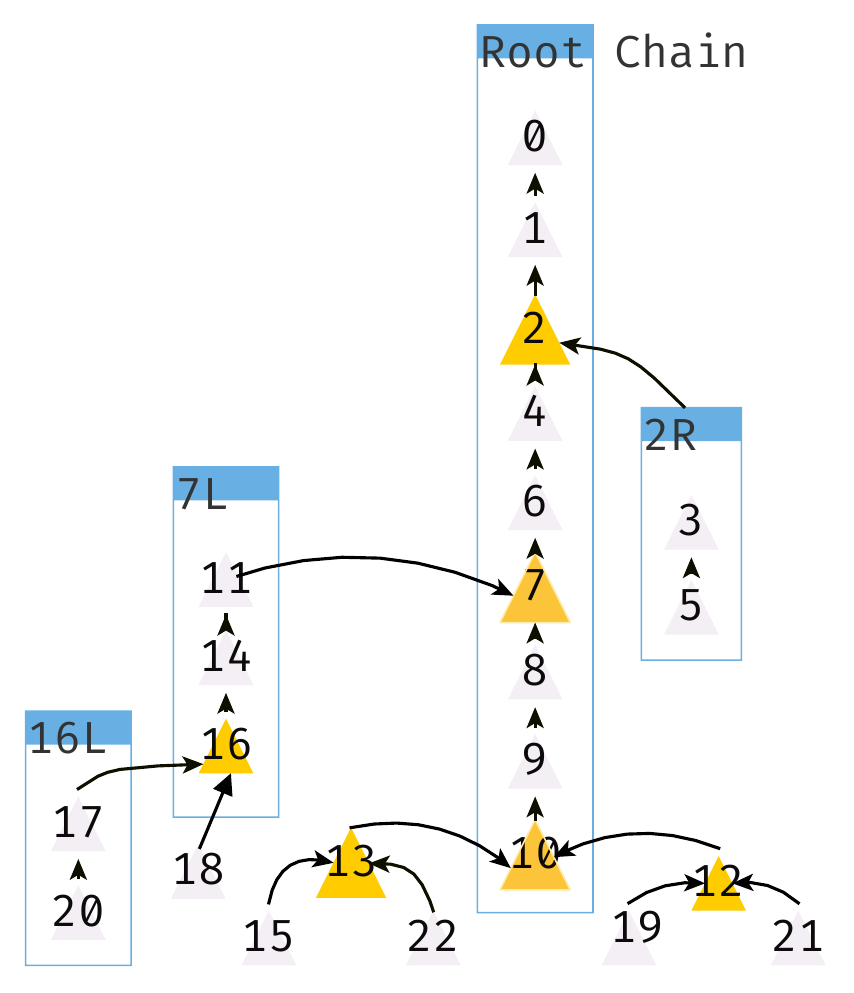}}
  \caption{\label{fig:yed:expansion} \figcaptsize The process of expanding the complete dendrogram from the contracted dendrogram.
%
%
1. In \Cref{fig:yed:exp0}, the $\alpha$ dendrogram is shown, with triangles
representing $\alpha$ edges and circles representing $\alpha$ vertices. All non-$\alpha$
edges are displayed within their corresponding $\alpha$ vertices.
2. First, we identify non-$\alpha$ edges that belong to an $\alpha$ leaf
chain by comparing it to the parent of the $\alpha$
vertex containing the edge (\Cref{fig:yed:exp0}). Any edge with a higher
index than the $\alpha$ parent is considered part of the $\alpha$ leaf chain.
3. We then check the non-$\alpha$ edges that are not part of any leaf chains to
see if they belong to a $\beta$ leaf chain. To do this, we compare each
non-$\alpha$ edge to the $\beta$ edge. The $\beta$ edge is the parent of the
$\alpha$ parent we identified in the previous step. Any edge with a higher index
than the $\beta$ parent is marked as part of its $\beta$ leaf
chain(\Cref{fig:yed:exp2}). This process continues until all edges are assigned
to a leaf chain of some level, or there are no more contraction levels
remaining.
4. Any unassigned edges are allocated to the root chain if further contraction
is not possible, as depicted in \Cref{fig:yed:exp3}.
5. Finally, each chain is sorted to form partial dendrograms. These partial
dendrograms are then merged to produce the final dendrogram, as shown in
\Cref{fig:yed:exp4}.
}
\end{figure*}

\subsection{Recursive tree contraction}
\label{sec:rtreecontraction}
Pandora constructs a condensed version of \mst by contracting all edges except
the $\alpha$ edges. The dendrogram of this condensed MST is identical to the one
obtained by merging chain and leaf nodes in the full dendrogram, proved in
\Cref{sec:correctness}. This simplified dendrogram effectively captures the full
dendrogram's structure.

\textbf{Computing $\alpha$-Edges:} 
An $\alpha$-edge is a type of edge-node that has two children that are also
edge-nodes.  If an edge-node $e_{k}=\Edge{v}{u}$ has a vertex node as a child,
it will be either $v$ or $u$. The parent of $v$ is given by $P(v) =
\maxIncident{v}$. In case $k$ is not equal to $\maxIncident{v}$ and
$\maxIncident{u}$, then $e_{k}$ is not a parent of either vertex node incident
on it. This means that both its children are edge-nodes. Therefore, an edge-node
$e_{k}=\Edge{v}{u}$ is an $\alpha$-edge if: 
\begin{equation}
    \label{eq:alphaedge}
k  \neq  \maxIncident{v} \ \; \text{and}\; k  \neq  \maxIncident{u}. 
    \end{equation}
\Cref{eq:alphaedge}  allows for the identification of all
$\alpha$-edges using a constant-time operation for each edge.

In \Cref{fig:yed}, we demonstrate the process. Let's look at the Minimum
Spanning Tree (MST) example in \Cref{fig:yed:MST10}. In \Cref{fig:yed:MST15}, we
highlight the $\alpha$-edges of the MST. For example, $e_{16}=\Edge{i}{d}$ is an
$\alpha$-edge because $\maxIncident{i} = 20$ and $\maxIncident{d} = 18$, which
are both different from 16. None of the terminal edges are
$\alpha$-edges. For instance, $e_{1}=\Edge{m}{k}$ is not an $\alpha$-edge
because $\maxIncident{m} = 1$. Additionally, several internal edges like
$e_{20}$ and $e_{17}$ are also non-$\alpha$ edges.


\textbf{Computing $\alpha$-\mst:}
We first identify the $\alpha$ edges in the original tree. Then, we contract the
remaining non-$\alpha$ edges to create a new tree called $\alpha$-\mst
($T_{\alpha}$). In $T_{\alpha}$, each vertex is an $\alpha$-vertex, representing
multiple vertices from the original tree that have been contracted. We also keep
track of the mapping between the original vertices and their counterparts in
$T_{\alpha}$, which is important for tracing back to the original structure.

For example, in \Cref{fig:yed:MST15}, we show an MST with highlighted $\alpha$
edges. We contract the non-$\alpha$ edges to obtain the contracted tree shown in
\Cref{fig:yed:MST20}. The vertices in \Cref{fig:yed:MST15} that are merged to
form a supervertex are colored with the same color. In \Cref{fig:yed:MST20},
vertices $a$, $n$, $o$, and $p$ are merged together to form a single
supervertex, all colored cyan.

\textbf{Multilevel tree contraction:} 
The dendrogram of $T_{\alpha}$ can be computed by recursively applying the same
edge contraction strategy. This leads to a multilevel tree contraction scenario.
In each iteration, we find $\alpha$-edges present at that level of the
contracted tree, and contract the remaining edges to get the tree for the next
iteration. The recursion stops when there are no more $\alpha$-edges left. At
this point, we get a single chain dendrogram, obtained by sorting the edges
according to their indices.

\Cref{fig:yed:MST30} illustrates the $\beta$-\mst, which emerges from the second
contraction level applied to the $\alpha$-\mst depicted in \Cref{fig:yed:MST20}.
The $\beta$-\mst can no longer be contracted, hence the recursion concludes. The
final contraction stage is represented by the $\beta$-dendrogram, displayed in
\Cref{fig:yed:dgram30}.

To summarize, we begin with a complete Minimum Spanning Tree (MST). We create a
series of smaller trees by performing multilevel tree-contractions on this tree.
We continue this process until we have a tree without any $\alpha$ edges. The
dendrogram of this tree forms a single chain, which we can obtain by sorting.
This results in a highly compact dendrogram. In the next section, we explain how
to expand this condensed dendrogram into a comprehensive one.


\subsection{Efficient dendrogram expansion}
\label{sec:expansion}
In this section, we will explain how to construct a complete dendrogram from a
condensed dendrogram, which we call \emph{expansion}. Then, we explain the
expansion process for a single-level contraction in \Cref{sec:expansion:alpha}. However, single-level
contraction is not optimal for reconstructing a dendrogram. Therefore, we have
developed an expansion algorithm that utilizes all contraction levels, described in \Cref{sec:expansion:optimal}.


\begin{figure}[tbp]
  \includegraphics[width=\columnwidth]{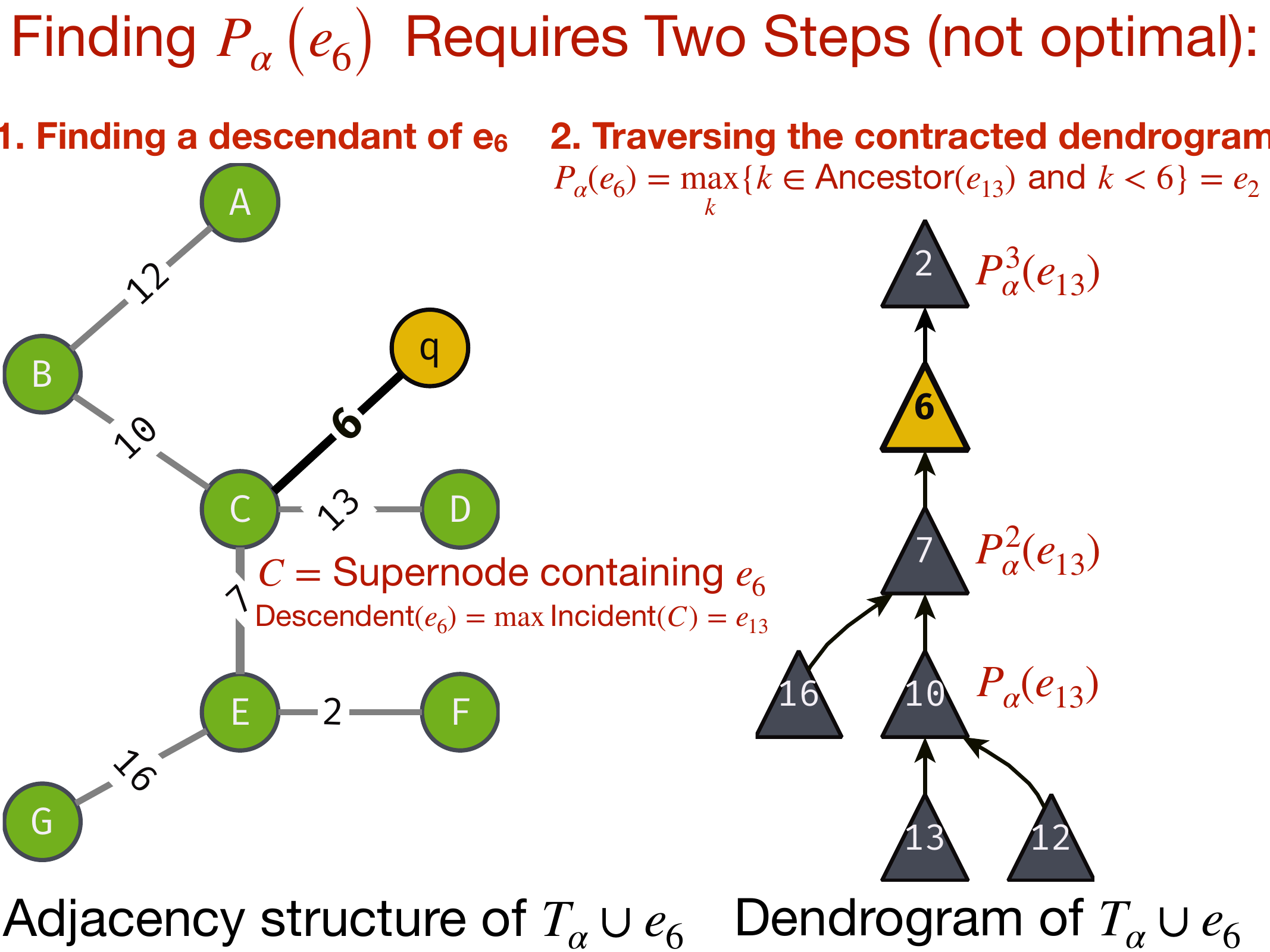}
  \caption{\label{fig:e6insert}\figcaptsize Inserting a non-$\alpha$
  edge $e_{6}$ into the $\alpha$-dendrogram. In this process, a single level
  contraction is done within the supervertex to find the parent of the edge. For
  example, in the case of edge $e_6$, we identify its parent by finding the
  maximum incident edge of the supervertex $C$, which is $e_{13}$. We consider
  $e_{13}$ as a descendant of $e_6$. To locate the parent of $e_6$, we go
  through the dendrogram upwards and select the ancestor with the highest index
  among all ancestors of $P_\alpha(e_{13})$. This way, we determine that the
  parent of $e_6$ is $e_2$, represented in the dendrogram of $T_\alpha \cup
  e_6$. However, this accurate method can be inefficient because it may require
  traversing the entire dendrogram. }
\end{figure}

\subsubsection{Dendrogram expansion from single-level tree contraction}
\label{sec:expansion:alpha}
Given an input Minimum Spanning Tree (MST) called $T$, a contracted tree
containing all the $\alpha$ edges called $T_\alpha$, and the dendrogram of
$T_\alpha$ specified with the parent-child relation $P_\alpha$, our objective is
to assign each non-$\alpha$ edge to a specific dendrogram chain. To accomplish
this, we follow these steps:

\begin{enumerate}
    \item Find the $\alpha$-vertex $\alphaVertexOf{e}$ containing $e$.
    \item Determine the dendrogram parent of $\alphaVertexOf{e}$ in $\alpha$-dendrogram: $P_\alpha(\alphaVertexOf{e}).$
    \item Traverse the $\alpha$-dendrogram to find the $P_{\alpha}(e)$: Starting from $P_\alpha(\alphaVertexOf{e}).$
     and traverse the dendrogram upwards until an $\alpha$ edge with a smaller index than $e$ is encountered.
\end{enumerate}

Let's consider how to map the edge $e_6$ into the dendrogram for the minimum
spanning tree (MST) shown in \Cref{fig:yed}. The $\alpha$-vertex that contains
$e_6$ is denoted as $\alphaVertexOf{e_6} = C$ in \Cref{fig:yed:MST20}. The
parent of $\alphaVertexOf{e_6}$ in the $\alpha$-dendrogram is
$P_\alpha(\alphaVertexOf{e_6}) = e_{13}$ shown in \Cref{fig:yed:dgram20}. To
find the parent of $e_6$, we traverse the alpha dendrogram from bottom to top,
starting at $e_{13}$. We look for an $\alpha$ edge with a lower index than
$e_6$(see \Cref{fig:e6insert}). In this case, the lower-indexed edge is $e_2$,
which becomes the $\alpha$-parent of $e_6$. Since $e_7$ is placed on the left
side of $e_2$, we assign $e_6$ to the chain $2L$.

However, this method is not optimal as it requires traversing the alpha
dendrogram in a bottom-up order for all non-$\alpha$ edges. In the worst case,
the height of the alpha dendrogram tree can be $\bigo{n}$. Consequently, finding
chains for all non-$\alpha$ edges would require $O(n^2)$ work.

\subsubsection{Efficient dendrogram expansion from multilevel tree contraction}
\label{sec:expansion:optimal}

We can optimize the dendrogram expansion process by making two key observations.

First, we can quickly identify edges that are part of a leaf chain without
traversing the entire $\alpha$ dendrogram. Second, for edges that are not in a
leaf chain of the $\alpha$ dendrogram, we can efficiently check if they are in a
leaf chain of the $\beta$ dendrogram. 

By recursively applying this process, we can associate all edges with a leaf
chain at some level. Instead of traversing the $\alpha$ dendrogram from the bottom
up, which can be inefficient due to its height, we start checking for leaf chain
membership at level of dendrograms. This approach is more efficient since the
number of contraction levels  \(\log_{2} n\).

\textbf{Leaf Chains:}
Leaf chains are linked to their respective dendrograms.  An $\alpha$ leaf chain
refers to a sequence that concludes with a leaf edge in the dendrogram. For
example, in \Cref{fig:yed:exp4}, the sequence denoted by $16L$ qualifies as an
$\alpha$ leaf chain. Upon removing all $\alpha$ leaf chains from a dendrogram,
new leaf chains emerge with an $\alpha$ edge as terminal, termed $\beta$ leaf
chains. A $\beta$ leaf chain may encompass multiple $\alpha$ chains that are not
leaves, and the $\alpha$ edges linked to these chains become part of the $\beta$
chain. These $\alpha$ and non-$\alpha$ edges together create an unbroken lineage
within the full dendrogram. This concept of leaf chains can be extended to
higher-level contractions as well.

\textbf{Mapping edges to a leaf chain:} 
To construct the dendrogram efficiently, we utilize a constant-time method to
determine if an edge is part of a leaf chain at any level. We aim to identify
the earliest contraction level at which each edge becomes part of a leaf chain.

For each non-$\alpha$ edge $e$, we first check if it belongs to an $\alpha$ leaf
chain by comparing the index of the $\alpha$ parent of $\alphaVertexOf{e}$ to
the index of $e$. If the $\alpha$ parent's index is lower, $e$ is part of an
$\alpha$ leaf chain. If not, we check for inclusion in a $\beta$ leaf chain by
examining the $\beta$ parent of $\betaVertexOf{e}$ in the $\beta$ dendrogram and
comparing it to $e$. Determining an edge's leaf chain membership at any level
takes constant time ($O(1)$).

To map non-$\alpha$ edges to their respective chains, we check if they are part
of an $\alpha$ leaf chain. If not, we then check if they belong to a $\beta$
leaf chain, and so on, until the edge is placed in a chain. Any edges not
assigned to a chain at the final level are grouped together in the root chain.
The maximum number of contraction levels determines the cost of associating an
edge with a  leaf chain.

\textbf{Example:}
To map a non-$\alpha$ edge, such as $e_{15}$, to a chain, we first identify its
$\alpha$-vertex, $C$, as shown in Figure \ref{fig:yed:MST20}. The
$\alpha$-parent of $C$, $P_{\alpha}(C)$, is 13 (Figure \ref{fig:yed:dgram20}).
Since 15 is greater than 13, $e_{15}$ is part of the leaf chain associated with
$e_{13}$, specifically the $13R$ chain.

Next, consider edge $e_{11}$, with $\alpha$-vertex $E$ and $\alpha$-parent
$P_{\alpha}(E) = 16$. As 16 is greater than 11, $e_{11}$ is not in the leaf
chain. We then determine if it's part of a $\beta$-leaf chain. The
$\beta$-vertex containing $e_{11}$ is $X$ (Figure \ref{fig:yed:MST30}), with
$\beta$-parent $P_{\beta}(X) = 7$ (Figure \ref{fig:yed:dgram30}). Since 11 is
greater than 7, $e_{11}$ is indeed in a $\beta$-leaf chain.

The mapping process is illustrated in Figure \ref{fig:yed:expansion}. We start
with the $\alpha$-dendrogram (Figure \ref{fig:yed:dgram20}), determine the
$V_{\alpha}$ for all edges (Figure \ref{fig:yed:exp0}), and identify those in an
$\alpha$ leaf chain (Figure \ref{fig:yed:exp1}). Edges not in an $\alpha$ leaf
chain are then checked against $\beta$ leaf chains (Figure \ref{fig:yed:exp2}).
Finally, edges not in a $\beta$ leaf chain are assigned to the root chain
(Figure \ref{fig:yed:exp3}).

\subsubsection{Final dendrogram construction}
In the previous step, we assigned all the edges to a leaf chain or root chain.
Now, we will use this information to build the entire dendrogram. This process
involves two main steps:

\textbf{Sorting the Chains:} Sorting each chain forms partial dendrograms. In
the sorted chain, we assign the parent of each edge to its predecessor in the
sorted chain, except for the first edge in the chain, which is handled in the
next step.

\textbf{Stitching the Chains}: Each chain is a leaf chain of a contracted edge
from a certain level of contraction. The parent of the first edge in the chain
is marked as the corresponding edge for that chain. For instance, the
$\alpha$-leaf chain $(17,20)$ corresponds to the $\alpha$-edge $e_{16}$
(\Cref{fig:yed:exp1}), while the $\beta$-leaf chain $(16,11,14)$ corresponds to
the $\beta$-edge $e_{7}$ (\Cref{fig:yed:exp2}). Therefore, the parent of
$e_{17}$ in the $\alpha$-leaf chain $(17,20)$ is the $\alpha$-edge $e_{16}$, and
the parent of $e_{11}$ is the $\beta$-edge $e_{7}$, shown in
\Cref{fig:yed:exp4}. By connecting chains in this way, the complete dendrogram
is formed shown in \Cref{fig:yed:exp4}.

\subsection{Correctness of the \NewAlgName Algorithm}
\label{sec:correctness}
To establish the correctness of our algorithm, we first define the
\textit{Lowest Common Dendrogram Ancestor (LCDA)}, which represents the earliest
shared ancestor edge in the dendrogram for any two edges. We then introduce a
theorem that links the LCDA to the structure of the minimum spanning tree,
asserting that the LCDA of two edges is equivalent to the edge with the smallest
index on the path that connects them in the tree.

Subsequently, we explore the impact of tree contraction on the dendrogram's
structure. We introduce the concept of \textit{dendrogram lineage preserving
tree contraction}, a type of tree contraction that retains the dendrogram
ancestry properties of the original tree. We outline the necessary and
sufficient conditions for a tree contraction to be considered
hierarchy-preserving.

Finally, we validate our algorithm for generating dendrograms through tree
contraction. We demonstrate that contracting $\alpha$-edges meets the conditions
necessary to ensure the algorithm's correctness.

\begin{definition}[Path]
  For any given tree \( T \), a path connecting two edges \( e_i \) and \( e_j
  \), referred to as \( \text{Path}(e_i, e_j) \), is a sequence of edges \(
  \{e_{a1}, e_{a2}, \ldots, e_{an}\} \) present in \( T \) such that \( e_{a1} =
  e_i \), \( e_{an} = e_j \), and for every \( k \) (where \( 1 \leq k < n \)),
  the edges \( e_{ak} \) and \( e_{a(k+1)} \) are adjacent in \( T \).
  \end{definition}

\begin{definition}[Ancestors of an Edge]
  For any edge $e\in T$,  \( \text{Ancestors}(e) \) denotes its set of all
  ancestors including itself. If \( P(e)\) is the dendrogram parent of edge $e$,
  then \( \text{Ancestors}(e) \) is defined  follows:
  \[
  \text{Ancestors}(e) = \cbra{e, P(e), P^2(e), \ldots, P^k(e)}
  \]
  where $P^k(e)$ is the dendrogram root.
  \end{definition}
  
\begin{definition}[\lcdaexpanded ]
  \label{def:lcda}
  In a dendrogram $T$, the \lcdaexpanded~of edges $e_i$ and $e_j$ (denoted as
  $\LCDA{e_i}{e_j}$) is the deepest edge in $T$ that is an ancestor of both
  $e_i$ and $e_j$, with each edge considered an ancestor of itself.
\end{definition}

\begin{theorem}
\label{thm:path_lcda}
Let \( e_i \) and \( e_j \) be any two edges in a tree \( T \),  then \(
\LCDA{e_i}{e_j} \) is the heaviest edge in the path \( \text{Path}(e_i,
e_j) \). 
\[
\LCDA{e_i}{e_j} =  \text{Heaviest Edge  in the}\ \text{Path}(e_i, e_j)
\]
In other words, if edges are sorted by weights in descending order in
$T$,  then \( \LCDA{e_i}{e_j} \) has the smallest numerical index among
all edges in path \( \text{Path}(e_i, e_j) \). Formally,
\[
\LCDA{e_i}{e_j} = e_{k}\ \text{where}\ k=\underset{e \in \text{Path}(e_i, e_j)}{\arg\min} \text{Index}(e).
\]
\end{theorem}  
\begin{proof}
  Let \( e_h \) denote the heaviest edge in the minimal component \(
  \mathcal{C}_{ij} \) containing both \( e_i \) and \( e_j \). We will prove
  that \( e_h \) is the \(\LCDA{e_i}{e_j}\) and that \( e_h \) is also
  the heaviest edge in the path \( \text{Path}(e_i, e_j) \).
  
  \textit{Part 1: \( e_h \) is \(\LCDA{e_i}{e_j}\).}
  Since \( e_h \) is the heaviest edge in \( \mathcal{C}_{ij} \), by the
  properties of the top-down dendrogram construction, \( e_h \) is the root of
  dendrogram subtree corresponding to  \( \mathcal{C}_{ij} \), hence  must be
  an ancestor to all edges in \( \mathcal{C}_{ij} \), including \( e_i \) and
  \( e_j \). Therefore, \( e_h \) is a common ancestor of \( e_i \) and \( e_j
  \).

  To show that \( e_h \) is the \textit{last} common ancestor, assume for
  contradiction that there exists an edge \( e_l \) that is a common ancestor
  of \( e_i \) and \( e_j \) and is lower in the tree than \( e_h \). The
  existence of \( e_l \) implies the top-down process produced a component
  smaller than \( \mathcal{C}_{ij} \) that contains both $e_i$ and $e_j$, thus 
  contradicting the minimality of \( \mathcal{C}_{ij} \). Therefore, \( e_h \)
  is the \(\LCDA{e_i}{e_j}\).
  
  \textit{Part 2: \( e_h \) is the heaviest edge in \( \text{Path}(e_i, e_j) \).}
\( e_h \) must be in \( \text{Path}(e_i, e_j) \); otherwise, its removal
  would not disconnect \( e_i \) from \( e_j \), contradicting the minimality
  of \( \mathcal{C}_{ij} \). Since \( e_h \) is the heaviest edge of \(
  \mathcal{C}_{ij} \) and \( e_h \in \text{Path}(e_i, e_j) \), it follows that
  \( e_h \) is the heaviest edge in \( \text{Path}(e_i, e_j) \).
  
  Thus, \( e_h \) is both the \(\LCDA{e_i}{e_j}\) and the heaviest edge
  in \( \text{Path}(e_i, e_j) \), which completes the proof.
  \end{proof}

\begin{corollary}
\label{cor:incident_ancestor}
If two edges \( e_i \) and \( e_j \) are incident in a tree \( T \), then one
of them is an ancestor of the other in dendrogram .
\end{corollary}

\begin{proof}
If two edges \( e_i \) and \( e_j \) are incident, then  \( \text{Path}(e_i,
e_j) =\{e_i, e_j \}\). From~\Cref{thm:path_lcda}, we know that \(
\LCDA{e_i}{e_j} \in \text{Path}(e_i, e_j) \). Therefore, the $\LCDA{e_i}{e_j}\in
\{e_i, e_j \}$. The $\LCDA{e_i}{e_j}$ being an ancestor of both edges makes one
edge an ancestor of the other.
\end{proof}

\subsubsection{ Dendrogram of a contracted tree}
\label{sec:DendroEdgeContraction}
In this section, we seek to establish how edge contraction affects the
dendrogram. Specifically, we will show that under certain edge contractions
preserve the ancestory properties in the dendrogram of the original tree. We
characterize these contractions by providing neccessary and sufficient
conditions for the edge contractions.

Consider a tree \( T = (V, E) \) with its dendrogram defined by a parent
function \( P \). Let \( T_c = (V_c, E_c) \) be a contraction of \( T \) with
its contracted dendrogram defined by the parent function \( P_c \).

\begin{theorem}
  \label{thm:edge_contraction_dendrogram}
  For any two edges \( e_i, e_j \) in \( T_c \), if \( e_i \) is an ancestor of
  \( e_j \) in the dendrogram of \( T \), then \( e_i \) is also an ancestor of
  \( e_j \) in the dendrogram of \( T_c \). However, the converse is not
  necessarily true.
  \end{theorem}
  
  \begin{proof}
    The path between $e_i$ and $e_j$ in the tree \( T_c \) is either
a subset of or equal to their path in $T$ due to edge contraction.
Therefore, if $e_i$ is the heaviest edge in the path between $e_i$ and
$e_j$ in $T$, then it is also the heaviest edge in the path between $e_i$
and $e_j$ in $T_c$. This implies that the LCDA of
$e_i$ and $e_j$ in $T_c$ is $e_i$, and $e_i$ is an ancestor of $e_j$
in the dendrogram of $T_c$.
  
  
  
  
  \end{proof}
\subsubsection{Dendrogram Lineage Preserving Tree Contraction}
  The converse of~\Cref{thm:edge_contraction_dendrogram} does not hold for all
  edge contractions. This means that if \( e_i \) is an ancestor of \( e_j \) in
  the dendrogram of an arbitrary edge contraction \( T_c \), it is not necessary
  that \( e_i \) is also an ancestor of \( e_j \) in the dendrogram of \( T \).
  Therefore, an arbitrary edge contraction cannot accurately reconstruct the
  dendrogram of \( T \) from the contracted dendrogram. This limitation leads to
  the development of a stricter edge contraction referred to as \emph{Dendrogram
  Lineage Preserving Tree Contraction}, which allows for the reconstruction of
  \( T \)'s dendrogram from \( T_c \).

  \begin{definition}[Dendrogram Lineage Preserving Tree Contraction]
  \label{def:lineage_preserving_contraction}
  A dendrogram lineage preserving tree contraction refers to an edge contraction
  \( T_c = (V_c, E_c) \) of a tree \( T = (V, E)\) where, for any two edges \(
  e_i, e_j \) in \( E_c \), \( e_i \) is an ancestor of \( e_j \) in the
  dendrogram of \( T_c \) \textbf{if and only if} \( e_i \) is an ancestor of \( e_j \)
  in the dendrogram of \( T \).
  \end{definition}
  
  \begin{theorem}[Dendrogram Lineage Preservation in Edge Contraction]
  \label{thm:lineage_preserving_contraction}
  An edge contraction \( T_c = (V_c, E_c) \) of a tree \( T \) preserves
  dendrogram lineage if and only if for every pair of edges \( e_i, e_j \) in \(
  T_c \), the Lowest Common Dendrogram Ancestor in \( T \), denoted as \(
  \lcda[T]{e_i}{e_j} \), is also in \( E_c \).
    \end{theorem}
    
\begin{proof}
Let \( T = (V, E) \) be a tree and \( T_c = (V_c, E_c) \) be a contracted
version of \( T \) that satisfies the requirement of
\Cref{thm:lineage_preserving_contraction}. We aim to prove that if \( e_i \)
is an ancestor of \( e_j \) in \( T_c \), then it must also be an ancestor in
\( T \).

Assume, for the sake of contradiction, that there exists an edge \( e_i \) in \(
T_c \) that is an ancestor of \( e_j \) in the contracted dendrogram of \( T_c
\), but not in the dendrogram of \( T \). Since \( e_i \) can be an ancestor of
\( e_j\) if and only if $\lcda[T]{e_i}{e_j} = e_i$, it follows that \(
\lcda[T]{e_i}{e_j} \neq e_i \) and \( \lcda[T_c]{e_i}{e_j} =  e_i
\). Hence, our assumption implies that \( \lcda[T]{e_i}{e_j} \neq
\lcda[T_c]{e_i}{e_j} \).

By the definition of the contracted tree $T_{c}$, \( \lcda[T]{e_i}{e_j} \)
is in \( E_c \). From~\Cref{thm:path_lcda}, \( \lcda[T]{e_i}{e_j} \) is the
heaviest edge in the path \( \text{Path}_T(e_i, e_j) \) between \( e_i \) and \(
e_j \) in \( T \). Therefore, in \( T_c \), \( \lcda[T]{e_i}{e_j} \) must
also be the heaviest edge in the path \( \text{Path}_{T_c}(e_i, e_j) \), which
implies \( \lcda[T_c]{e_i}{e_j} = \lcda[T]{e_i}{e_j} \).

This contradicts implication of our assumption that \( \lcda[T]{e_i}{e_j}
\neq \lcda[T_c]{e_i}{e_j} \). Therefore, if \( e_i \) is an ancestor of
\( e_j \) in \( T_c \), it must also be an ancestor in \( T \), which concludes
the proof.
 \end{proof}

 \subsubsection{Correctness of the Algorithm}
  \label{sec:finalcorrect}
  The minimum spanning tree (MST) has three types of edges: leaf edges (zero
children), chain edges (one child), and $\alpha$ edges (two children). In our
contraction strategy, we only contract non-$\alpha$ edges, leaf, and chain
edges. The contracted tree contains all the $\alpha$ edges. With their two
children, only the $\alpha$ edges are capable of providing divergent paths
necessary for the lowest common dendrogram ancestor (LCDA) of any two edges not
including itself. Leaf edges, with no children, and chain edges, with only a
single lineage, lack this branching structure. Therefore, the contracted
$\alpha$-\mst contains all the edges that can be LCDA of two edges, both
distinct from itself. Hence, from the \Cref{thm:lineage_preserving_contraction}
the lineage is preserved in the contracted dendrogram.

While lineage preserving contraction is necessary for reconstructing the final
dendrogram,  it's not sufficient to ensure the correctness
of~\Cref{algm:pandora}. For example, the $\beta$-MST also preserves lineage, but
the $\beta$ dendrogram can't be used to reconstruct the original tree. 

Our \Cref{algm:pandora} inserts non-$\alpha$ edges into a contracted
$\alpha$-dendrogram. To ensure the final tree is accurate, it's important to
demonstrate that the edge contraction $E_\alpha \cup \{e_r, e_s\}$ (where $e_r$
and $e_s$ are non-$\alpha$ edges) also preserves lineage. This guarantees that
any two non-$\alpha$ edges that aren't ancestors of each other in the correct
dendrogram do not become ancestors in the constructed dendrogram. Since LCDA of
$e_r$ and $e_s$ can be either an $\alpha$ edge, $e_r$ or $e_s$ , all of them
contained in $E_\alpha \cup \{e_r, e_s\}$. Therefore, the tree contraction
containing $E_\alpha \cup \{e_r, e_s\}$ satisfies the conditions
of~\Cref{thm:lineage_preserving_contraction} for any pair of non-$\alpha$ edges
$e_r$ and $e_s$.

\section{Asymptotic work analysis}
\label{sec:complexity}
This section demonstrates that \NewAlgName~, is work-optimal. To establish this,
we first ascertain the lower bound of $\Omega(n \log n)$ for any dendrogram computing algorithm.
Subsequently, we show that \NewAlgName~achieve this lower bound.
\ifjournal
\textbf{\large Asymptotic lower bounds for dendrogram computation:}
\else
\subsection{Asymptotic lower bounds for dendrogram computation}
\fi
When computing a dendrogram, the number of operations required by any algorithm
must be at least the same as sorting a set of $n$ floats. However, while sorting
is required for bottom-up and \NewAlgName~methods, it is not mandatory. The
top-down approach(\Cref{alg:dgram-td}), for instance, can be carried out without
sorting. Therefore, we show that any algorithm for dendrogram computation must
have a lower bound that is no less than sorting.
\begin{theorem}
\label{thm:lowerbound} 
Any algorithm that computes a dendrogram from a tree with $n$ edges would
require at least $\Omega(n \log n )$ operations in the worst case.
\end{theorem}
\textbf{Proof:} If we have a list of $n$ floats that we want to sort, we can use
a minimum spanning tree $T$ with $n+1$ vertices and $n$ edges in a star
topology. This means that there is one central vertex connected to all other
vertices. By constructing this tree where the edge weight corresponds to the
floats in the list, we can compute its dendrogram, which consists of a single
chain of edges sorted by weight. We can then extract the sorted list of floats
from the dendrogram. It's worth noting that computing the dendrogram is the only
operation that requires more than $O(n)$ operations. If an algorithm could
compute this dendrogram with fewer than $\Omega(n \log n)$ operations, it would
mean that we could sort $n$ floats with lower than $\Omega(n \log n)$
complexity. However, this is a contradiction.

\ifjournal
\textbf{\large Asymptotic Work Complexity for \NewAlgName:}
\else
\subsection{Asymptotic Work Complexity for \NewAlgName}
\fi
Using \NewAlgName, we can construct the dendrogram of an \mst with $n$ edges in
$O(n \log n)$ operations. We establish bounds on the number of edges in each
level of contracted \mst, including leaf edges ($n_l$), chain edges ($n_c$), and
$\alpha$ edges ($n_{\alpha}$) as described in \Cref{sec:notation}. We show that
the number of $\alpha$-edges $n_{\alpha}$ is at most $(n-1)/2$ and the number of
contraction levels is at most $\ceil{\log_2(n+1)}$. Using these relationships,
we demonstrate that the total cost of edge contractions is of $O(n)$ and
dendrogram expansion costs $O(n \log n)$. Sorting is required twice, before and
after the tree contraction, and each operation costs $O(n \log n).$ Our
algorithm's overall cost is $O(n \log n)$, making it work-optimal.
\paragraph{Proof of $n_\alpha \leq (n-1)/2$:}
Firstly, it is worth noting that the number of leaf edges is always one more
than the number of $\alpha$-edges: $n_\alpha = n_l -1$. Formally, in a connected
rooted binary tree where each node has zero, one, or two children, the count of
leaf nodes (nodes with no children) is always one more than the count of nodes
having two children. This result can be proven easily by induction, and we omit
the proof here.

Secondly, the dendrogram has no edges besides $\alpha$-edges, leaf edges, and
chain edges. Therefore, $n_{\alpha} + n_{l} + n_{c} = n$. Additionally, since
$n_\alpha = n_l -1$, it follows that $n_{c} = n - 2n_\alpha -1$. Since the
number of chain edges is always non-negative, e.g. $n_c \geq 0$, we can conclude
that: 
\[ n_{\alpha} \leq (n-1)/2. \]

\paragraph{The number of contraction levels:}
Performing a single tree contraction will reduce the edges in the contracted
$\alpha$-\mst to no more than $(n-1)/2$. Similarly, after the second
contraction, the contracted $\beta$-\mst will have at most $(n_{\alpha}-1)/2$
edges, resulting in $n_{\beta} \leq (n-3)/4$ edges. 

With each subsequent contraction, the number of remaining edges decreases by
half. After $k$ levels of contraction, the remaining edges will be at most
$(n-2^{k}+1)/2^{k}$. Therefore, there will be no remaining edges after $k$
reaches $\ceil{\log_2(n+1)}$.

\paragraph{Cost of Edge Contraction:}
Performing tree contraction on a tree with $n$ edges is equivalent to a prefix
sum (or scan) operation on an array with $2n$ entries \cite{miller1985ptreecon}.
Hence, a single tree contraction requires fewer than $c_1n$ operations. for some
constant $c_1$. After the $k$-th contraction, the contracted tree will have at
most $(n-2^k+1)/2^k < n/2^k$ edges. Therefore, the total cost of contraction is
less than $\sum c_1 n/2^k  < 2c_1n$ operations.

\paragraph{Cost of Dendrogram Expansion:}
Mapping any edge node to its dendrogram chain costs $O(\log n)$. We check each
edge in the dendrogram prolongation to see if it is in a leaf chain in the
$k$-th level contraction, an $O(1)$ operation. Since there are $O(\log n)$
levels of contraction, mapping all $n$ edges to their respective chains leads to
an overall cost of $O(n \log n)$ for the dendrogram expansion.

\paragraph{Total Cost of Dendrogram Construction:}
We perform two sorting operations - before tree contraction and after mapping
edges to chains - in addition to tree contraction and dendrogram prolongation.
Both operations have a cost of $O(n \log n)$. As a result, our algorithm is
work-optimal with an overall cost of $O(n \log n)$.

\section{Performance portable implementation}\label{sec:implementation}

We implemented our algorithms as part of the ArborX~\cite{arborx2020} library.
We used Kokkos~\cite{kokkos2022}, a performance-portable programming model
which allows code to run on various CPU and GPU platforms without additional
modifications. Kokkos offers parallel execution patterns, such as parallel
loops, reductions, and scans, while abstracting hardware complexities. Kokkos
introduces "execution space" and "memory space" abstractions for execution and
memory resources. Kokkos does not perform hidden data copies, so users must
ensure data accessibility between memory and execution spaces. The Kokkos
library provides C++ abstractions and supports hardware through backends,
including CPUs and Nvidia, AMD and Intel GPUs.

The kernels in the described algorithms are all written as a sequence of
parallel loops, reductions and prefix sums, available in most parallel
libraries, e.g., Thrust. Kokkos uses internal heuristics to map a kernel to the
underlying hardware architecture.

The highest complexity kernel in the algorithm is the tree contraction. We use
a disjoint set method based on the union-find approach to contract multiple
vertices into a single super vertex. For our union-find implementation, we use
a synchronization-free GPU algorithm proposed in~\cite{jaiganesh2018}, which
uses pointer jumping.

Initially, an implementation based on Euler tour was considered. Euler tour
is more suitable for tree contractions as it requires a single prefix sum
kernel. This approach was used in~\cite{wang2021}. However, this would require
initial conversion of an MST, given as a set of edges, to the Euler tour. We
found the conversion algorithm~\cite{polak2021} to be expensive in practice,
taking time comparable to the full dendrogram construction.

\section{Experimental Results}\label{sec:results}

\subsection{Available open-source dendrogram construction implementations}
\label{sec:implementation_challenges}

\begin{table}
\caption{Available open-source dendrogram construction implementations}
\label{tab:implementations}
\begin{tabularx}{\columnwidth}{lX}
    \toprule
    Implementation & Description \\
    \midrule
    \texttt{scikit-learn} (Python)~\cite{scikit-learn} & Sequential implementation \\
    \texttt{hdbscan} (Python)~\cite{mcinnes2017} & Sequential implementation \\
    \texttt{hdbscan} (R)~\cite{hdbscan_r} & Sequential implementation \\
    Wang \emph{et al.}~\cite{wang2021} & Multi-threaded implementation in shared memory \\
    \texttt{rapidsai}~\cite{rapidsai} & Parallel MST implementation using GPUs, sequential dendrogram construction \\
    \bottomrule
\end{tabularx}
\end{table}

Most of the papers describing parallelization of \hdbscan or single-linkage
clustering focus on optimizing distance computations, such as efficiently
computing distance matrices or constructing MSTs, as we demonstrate
in~\Cref{tab:implementations}. We are aware of only one work~\cite{wang2021}
that explores hierarchy construction parallelization.
However, the online source code for~\cite{wang2021} only offers a sequential
implementation for the critical part of the algorithm, with parallelization
applied only to operations like sorting edges and finalization.
This approach is not unreasonable, as the sequential cost of sorting should be similar to or higher
than dendrogram construction.

The cuML library's single-linkage clustering uses RAPIDS.ai/Raft~\cite{rapidsai},
which only implements the parallel MST construction but not the dendrogram. We
have not yet found efficient parallel algorithms for dendrogram construction
beyond the work presented in \cite{wang2021}.


\begin{table}
\centering
\ra{1.2}
\footnotesize
\caption{Datasets used in experiments}
\label{tab:datasettable}
\begin{threeparttable}
  \begin{tabularx}{\linewidth}{ 
    >{\raggedright\arraybackslash}p{0.2\linewidth} 
    >{\raggedleft\arraybackslash}p{0.05\linewidth}  
>{\raggedleft\arraybackslash}p{0.05\linewidth}  
>{\raggedleft\arraybackslash}p{0.05\linewidth} 
>{\raggedleft\arraybackslash}p{0.05\linewidth}
    >{\raggedright\arraybackslash}p{0.3\linewidth} 
  }
  \toprule
{Name} & {Dim} & {$n_{pts}$} & {Imb} & {Ref.} & {Desc.}\\
\midrule
  Ngsimlocation3 & 2 & \(6\)M   & 1e3     & \cite{ngsim2018data}            & GPS loc \\
  RoadNetwork3   & 2 & \(400\)K & 150       & \cite{3droad}           & Road network \\
  Pamap2         & 4 & \(3.8\)M & 6e3     & \cite{reiss2012}        & Activity monitoring \\
  Farm           & 5 & \(3.6\)M & 5e4    & \cite{farm2024data}             & VZ-features\cite{varma2003} \\
  Household      & 7 & \(2.0\)M & 1e3     & \cite{UCI}              & Household power \\
  Hacc37M        & 3 & \(37\)M  & 1e5    & \cite{hacc}             & Cosmology \\
  Hacc497M       & 3 & \(497\)M & 6e5   & \cite{hacc}             & Cosmology \\
  VisualVar10M2D & 2 & 10M        & 3e3     & \cite{gan2017}          & GAN \\
  VisualVar10M3D & 3 & 10M        & 1e4    & \cite{gan2017}          & GAN \\
  VisualSim10M5D & 5 & 10M        & 43        & \cite{gan2017}          & GAN \\
  Normal100M2D   & 2 & 100M       & 1e5   & -                        & Random (normal) \\
  Normal300M2D   & 2 & 300M       & 4e5   & -                        & Random (normal) \\
  Normal100M3D   & 3 & 100M       & 4e5  & -                        & Random (normal) \\
  Uniform100M2D  & 2 & 100M       & 1e5     & -                        & Random (uniform) \\
  Uniform100M3D  & 3 & 100M       & 4e5   & -                        & Random (uniform) \\
\bottomrule
\end{tabularx}
\end{threeparttable}
\end{table}

\begin{figure*}
  \centering
    \includegraphics[width=\textwidth]{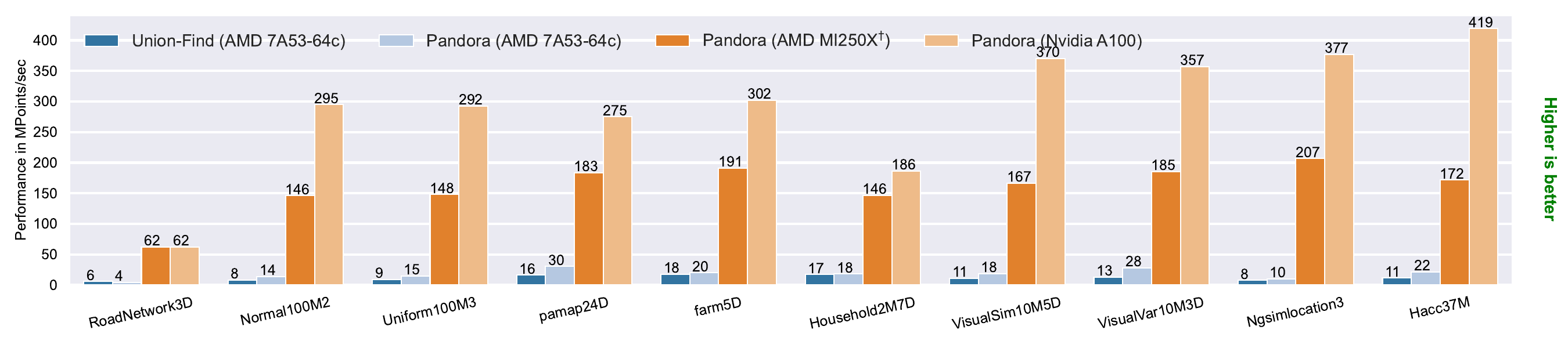}
   \caption{\label{fig:parallel-perf} Performance comparison of
   the multithreaded (using \amdcpu) and parallel (using \nvidiagpu and \amdgpu
   (single GCD)) implementations.}
\end{figure*}

\begin{figure}
  \includegraphics[width=\columnwidth]{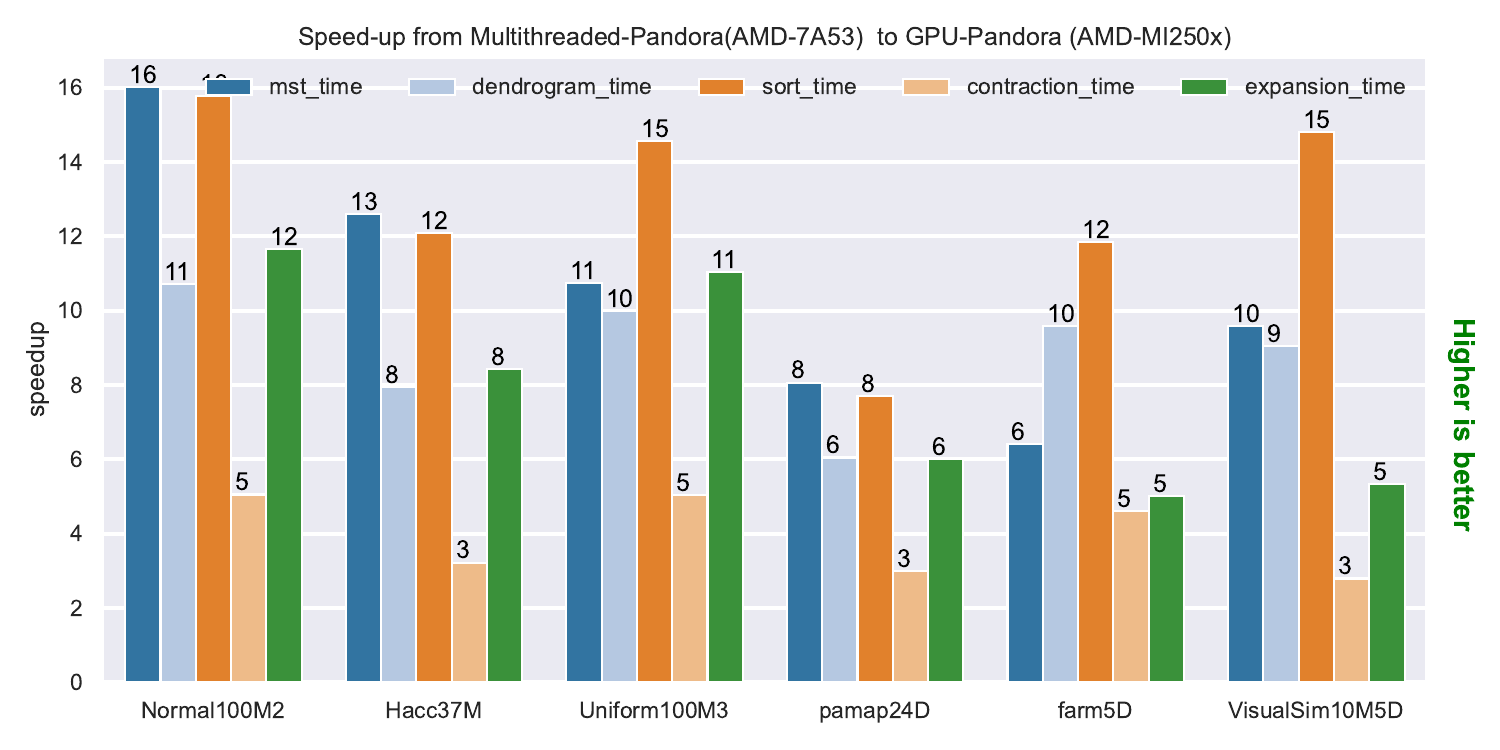}
  \caption{\label{fig:speedup} Speedup of \amdgpu over \amdcpu for different phases of HDBSCAN* 
  with\NewAlgName. }
\end{figure}
\begin{figure}
  \includegraphics[width=\columnwidth]{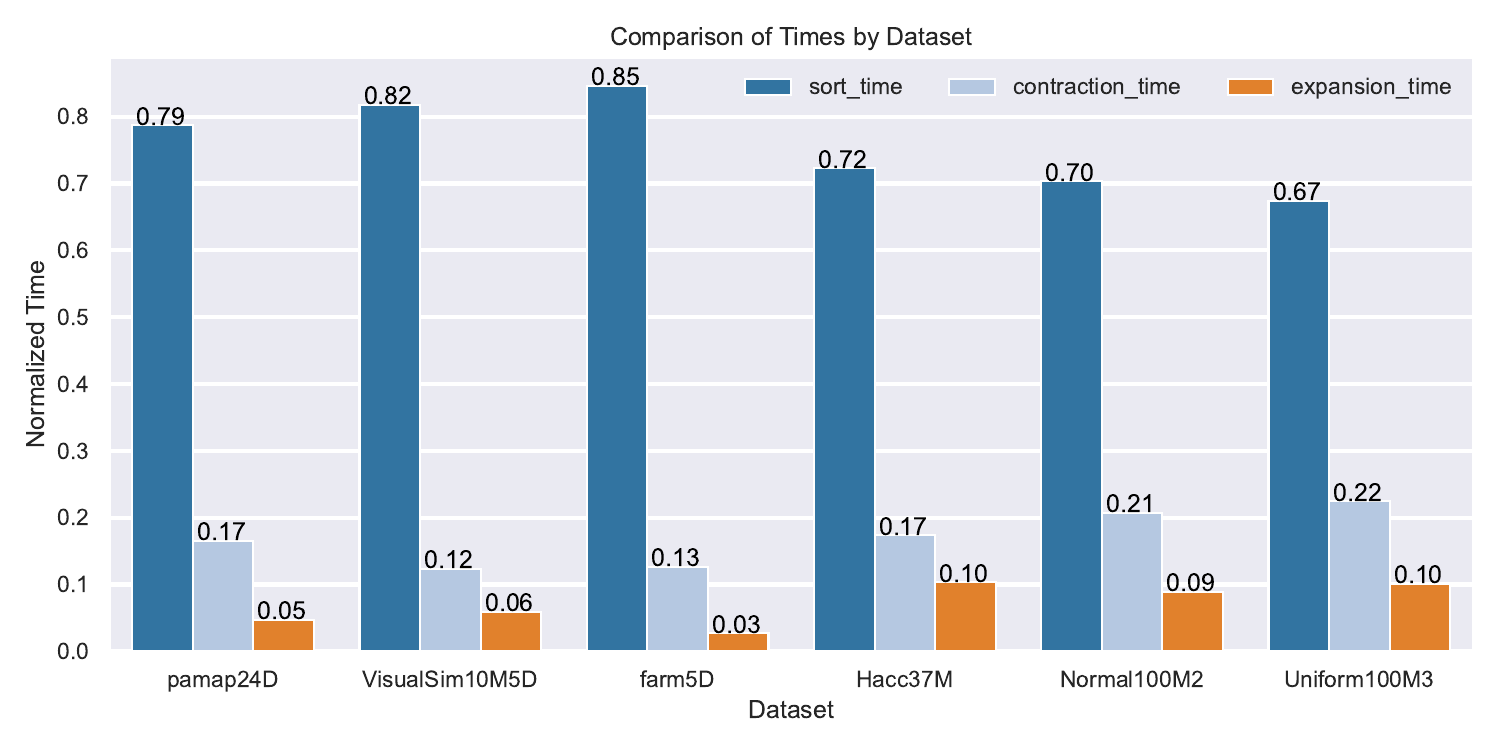}
  \caption{\label{fig:breakdown} Breakdown of time spent  \pandora
  on \amdcpu. }
\end{figure}

\begin{figure}
  \subfloat[Hacc497M\label{fig:samplingHacc}]{\includegraphics[width=0.49\columnwidth]{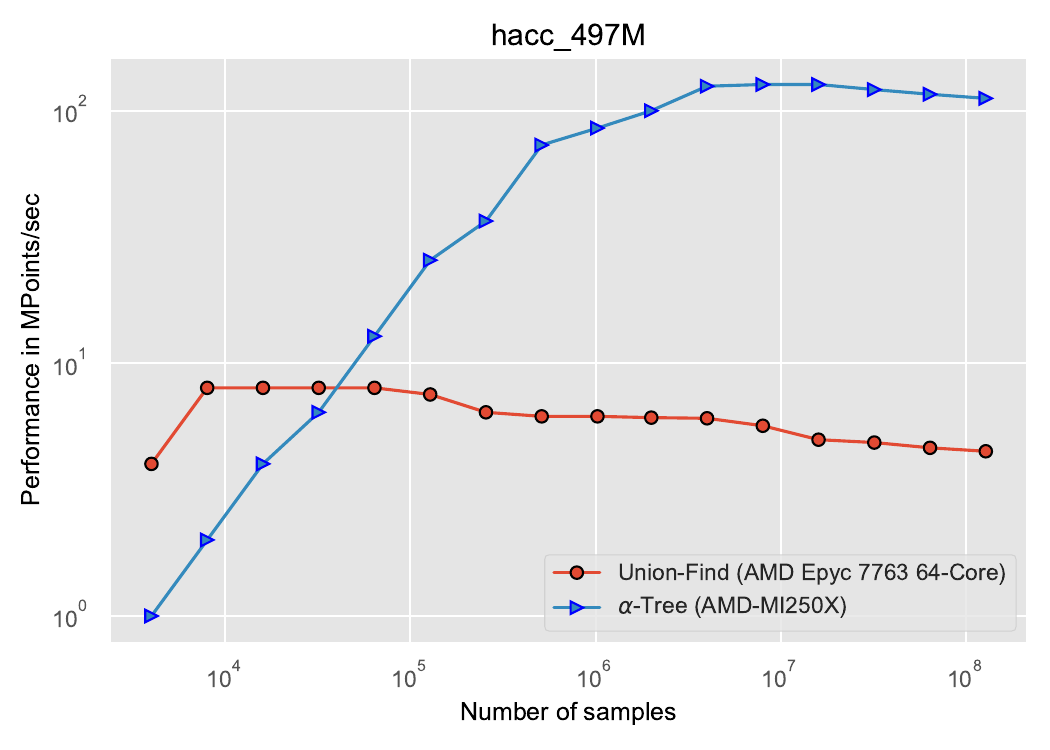}}
  \hfill
  \subfloat[Normal300M2\label{fig:samplingNormal}]{\includegraphics[width=0.49\columnwidth]{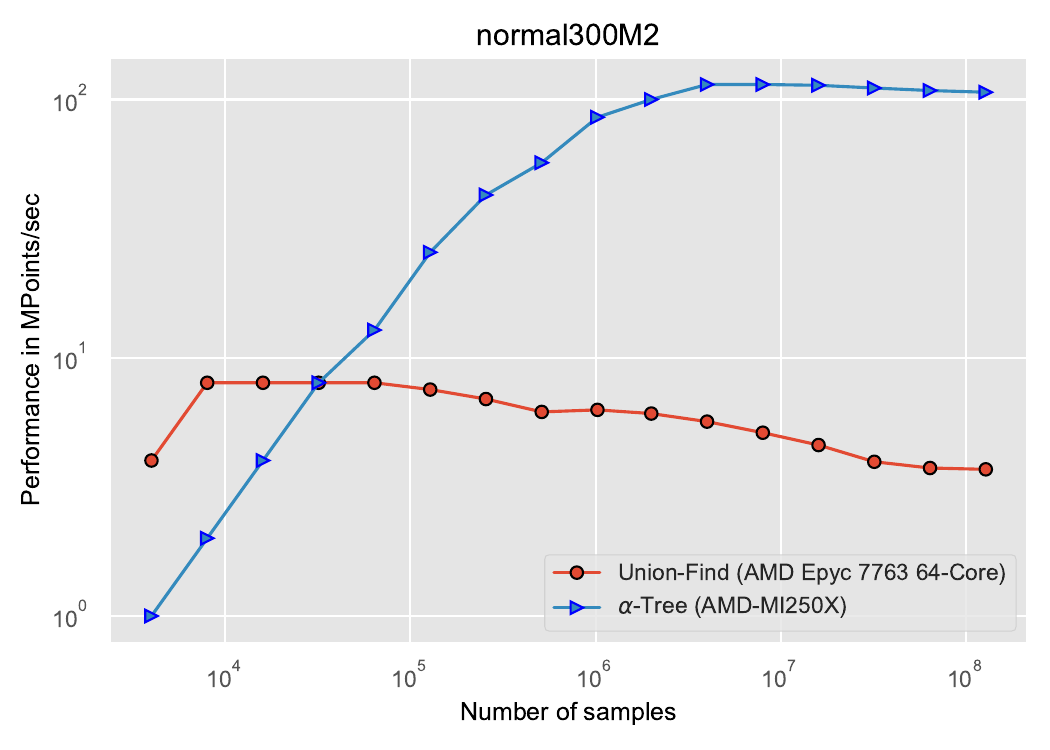}}
  \caption{\label{fig:sampling} Effect of the dataset size on the parallel performance using \amdgpu.}
  \end{figure}

  \begin{figure*}
    \subfloat[\label{fig:hacc37M} Hacc37M]{\includegraphics[width=0.49\textwidth]{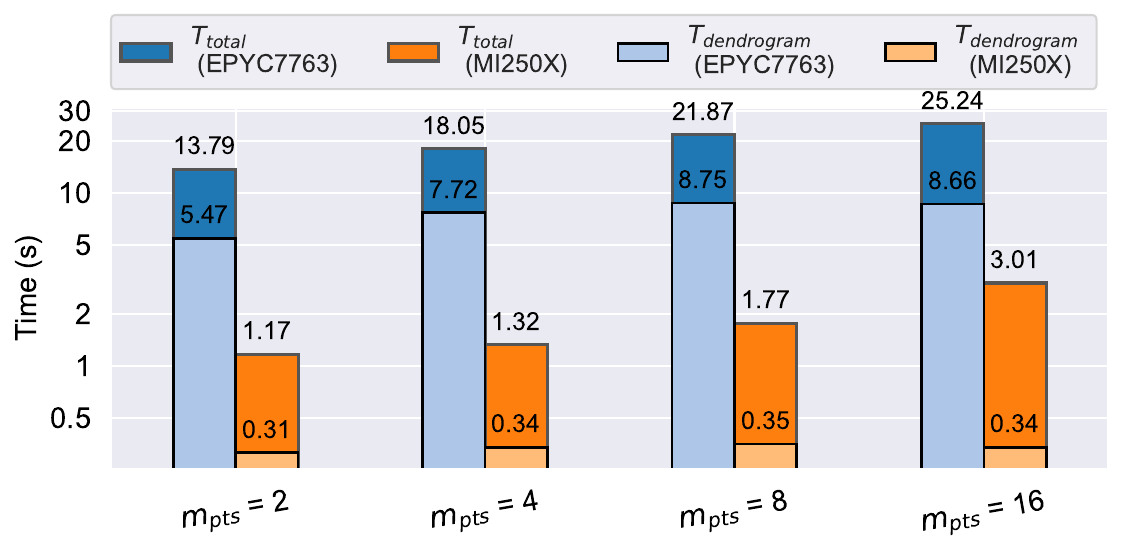}}
    \subfloat[\label{fig:Uniform100M3D} Uniform100M3D ]{\includegraphics[width=0.49\textwidth]{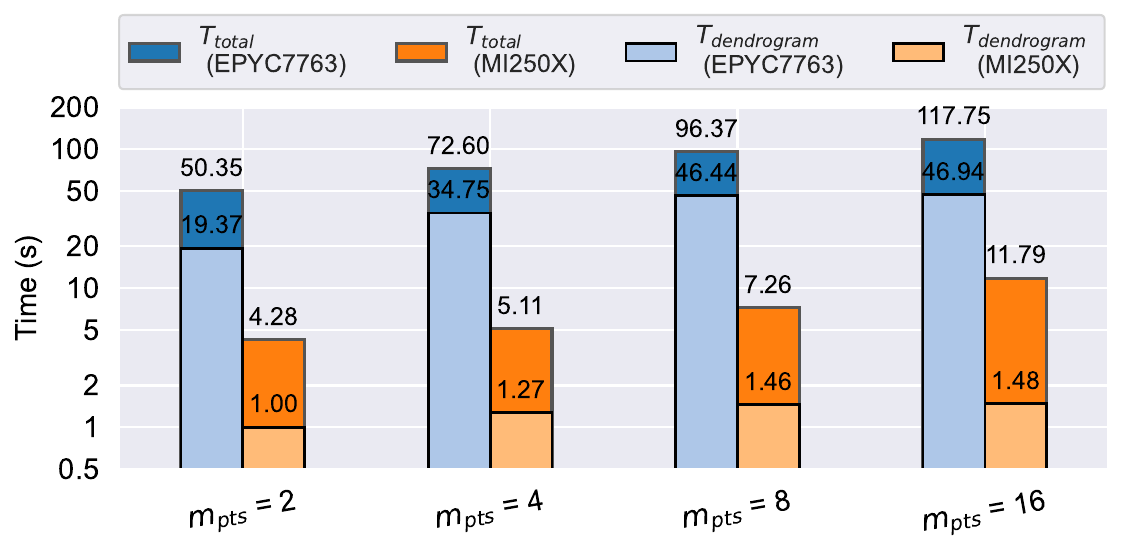}}
    \caption{\label{fig:hdbscan}Comparison of the time to compute first two steps of
      the \hdbscan algorithm using \wangemst on \amdcpu (blue) and ArborX+our dendrogram algorithm
      using \amdgpu (single GCD) (orange).}
    \end{figure*}

In our implementation, we used
\unless\ifblind
ArborX~\cite{arborx2020} (version 1.4-devel) to compute \emst, and
\fi
Kokkos~\cite{kokkos2022} (version 3.7) for implementing our parallel dendrogram algorithm.
\unless\ifblind
The implemented algorithm is available in the main ArborX
repository\footnote{\url{https://github.com/arborx/ArborX}}.
\fi

\subsection{Datasets}
For our experiments, we used a combination of artificial and real-world
datasets listed in~\cref{tab:datasettable} to comprehensively evaluate our
algorithm and meet our study goals. The GPS locations and HACC datasets
replicate real-world conditions, the datasets generated with~\cite{gan2017}
allow for better comparison with other works, and synthetic datasets help us
understand our algorithm's behaviour in different scenarios. We focused on 2D
and 3D datasets, where dendrogram construction was the bottleneck; on higher
dimensions, \mst construction was the primary issue.

\cref{tab:datasettable} also includes the information about the height of the
constructed dendrograms. Specifically, we show the ratio of a dendrogram height
for each dataset to that of a perfectly balanced binary tree. As we can see, it
clearly indicates the skewedness of the dendrograms, highlighting the challenge
of finding enough parallelization for an efficient algorithm.

\subsection{Testing environment}
The numerical studies presented in the paper were performed using \amdcpu (64 cores),
\nvidiagpu and a single GCD (Graphics Compute Die) of
\amdgpu\footnote{Currently, HIP (Heterogeneous-computing Interface for
Portability) -- the programming interface provided by AMD -- only allows the use
of each GCD as an independent GPU.}. The chips are based on TSMC's N7+, N7 and
N6 technology, respectively, and can be considered to belong to the same generation.
We used Clang 14.0.0 compiler for \amdcpu, NVCC 11.5 for \nvidiagpu, and ROCm
5.4 for \amdgpu.

\textbf{Competing Implementation}
We evaluate \pandora's performance on multithreaded \amdcpu, \nvidiagpu and
\amdgpu (single GCD). Our baseline is \unionfind from
\url{https://github.com/wangyiqiu/hdbscan}~\cite{wang2021}. Note that \unionfind
involves a parallel multithreaded sort and a sequential Union-find step phase.
We used this implementation for verification of our results. To our knowledge,
this is the fastest implementation of dendrogram computation available, and we
use it as a baseline for comparison. 

\noindent\textbf{Performance Metrics:} We measure our performance in
MPoints/sec, which represents the number of points (in millions) processed per
second. This metric is computed as 1e-6 * \#points in dataset/ Time to compute dendrogram.

\subsection{Performance evaluation of dendrogram construction}
We evaluate  the performance of the for the baseline \unionfind and \pandora on
various architecture across different datasets. The results are shown in
\Cref{fig:parallel-perf}.

\subsubsection{Multithreaded performance}
\label{sec:multithreaded-perf}
We found that \pandora outperforms \unionfind in multithreaded scenarios, with
speed-ups ranging from 0.66 to 2.2 times. The \dataset{RoadNetwork3D} dataset is
an exception, showing slower performance and having the smallest size with a
lower dendrogram imbalance. Smaller 2D datasets exhibit limited multi-threading
capabilities, with \pandora having a slight advantage. However, 3D and 4D
datasets show more significant speed-ups, reaching up to 2.2 times faster. In
higher-dimensional data sets, both algorithms have higher overall throughput,
but \pandora still has a slight edge. Thus, even with twice the sequential work,
\pandora remains faster than \unionfind in a multithreaded setting.

\subsubsection{GPU  performance}
\label{sec:parallel-perf}

Our study in \Cref{fig:parallel-perf} showcases \pandora's performance on both
\nvidiagpu and \amdgpu. Our findings indicate that \pandora operates
6-20$\times$ faster on \amdgpu (single GCD) than on \amdcpu, whereas \nvidiagpu outperforms
the fastest multithreaded variant by 10-37$\times$. The \dataset{RoadNetwork3D}
exhibits the lowest performance, caused by the small dataset size, not allowing
to reach GPU saturation. We generally observe higher speed-ups for
lower-dimensional datasets. This may only sometimes be the case as
higher-dimensional datasets often have lower-dimensional substructures,
resulting in similar behavior to lower-dimensional data. We also observe that
the GPU variant performs well for all ranges of dendrogram skewness. For
example, despite an imbalance of 43 in \dataset{VisualSim10M5D}, we still
achieve a considerable speed-up over multithreaded \pandora. We conclude that
\pandora has sufficient parallelism to utilize modern GPUs. , and
it works well with highly skewed and not-so-skewed dendrograms alike. 

Our implementation's performance is portable across multicore and GPU
architectures, with a single source compiled for various backends. We did not
optimize our algorithm specifically for any device architecture, nor did we
investigate the impact of architectural differences between \amdgpu and
\nvidiagpu on performance. However, we should note that we primarily utilized
\nvidiagpu in our software development, which may have led to a performance bias
towards that architecture.

\subsubsection{Scalability of different phases in \pandora} 
\label{sec:scale-phase}
\Cref{fig:speedup} shows the speed-up of \amdgpu over \amdcpu for different
phases of \hdbscan with \pandora. 
We divide the computation is divided into the following phases:
\begin{itemize}
  \item \emst construction
  \item Multilevel Contraction of  MST 
  \item reconstruction of the dendrogram from the contracted MST 
  \item Sorting (includes both initial and final sort, as well as other operations)
\end{itemize}
We observe that sorting is the most scalable phase, with a speed-up of
10-20$\times$ over \amdcpu. Where multilevel edge contraction is the least
scalable, with a speed-up of 3-5$\times$ over \amdcpu. 
In \Cref{fig:breakdown}, we show the breakdown of the time spent in different
phases of \pandora on \amdcpu. We observe that the time spent in sorting is
dominant, followed by multilevel edge contraction. The time spent in
reconstruction of the dendrogram is negligible. Therefore, despite poor scaling of 
multilevel edge contraction, the overall performance of \pandora is still
very good.





\subsubsection{Scaling problem sizes }
\label{sec:scaling}



One way to determine the effectiveness of a parallel algorithm is to understand
the smallest problem size at which it achieves peak performance. To this end, we
studied the performance of the \pandora algorithm with respect to the size of
the dataset. We randomly sampled a large dataset to maintain a given
distribution, as the algorithm could potentially be sensitive to the
distribution of the data. The results of sampling three datasets
(\dataset{Hacc497M}, \dataset{Normal300M2}, and \dataset{Uniform300M3}) on
\amdgpu are shown in \Cref{fig:sampling}. For reference, we also show the
performance of the \unionfind implementation on \amdcpu. For the \unionfind
implementation, performance immediately reaches its peak and slowly decreases
afterward. On the other hand, the performance of the new algorithm increases
with the number of samples until it reaches saturation and stays constant
afterward. At around 30,000 samples, the performance of \pandora-GPU exceeds
that of \unionfind. We can observe that GPU saturation occurs around the $10^6$
mark, which is typical for GPU algorithms.

\subsection{HDBSCAN$^{*}$ Performance} \label{sec:hdbscan}

\hdbscan (or hierarchical DBSCAN)~\cite{campello2015} is a density-based
hierarchical clustering algorithm that assigns points to the same cluster if
their local density is higher relative to a larger neighborhood. The algorithm
has the ability to find clusters of arbitrary shapes and varying densities and
has one main input parameter, $minPts$, which is the number of points used to
estimate local density. It uses a special metric called mutual reachability
distance, which is a variant of Euclidean distance complemented by the local
density.

\hdbscan consists of several steps: computing the MST using the mutual
reachability distance, computing a dendrogram for hierarchical clustering, and
(optional) computing non-overlapping (flat) clusters to convert hierarchical to
flat clusters and filter out noise points.
This paper focuses on improving the performance of the second step, which may
be the most expensive part of the algorithm for a low dimensional data, given
recent improvements in parallel MST computation on GPUs~\cite{prokopenko2022}.

\noindent\textbf{\hdbscan Parameters selection:} 
%
The only parameter relevant to our evaluation is $m_{pts}$, which is the number
of points to compute the core-distance. Different values of $m_{pts}$ produce
different dendrograms and affect the time spent on MST and dendrogram
construction. We use the default $m_{pts}=2$ in all our experiments except for
\Cref{fig:hdbscan} where we evaluate the performance of \hdbscan for different
values of $m_{pts}$. The performance gains of \pandora over \unionfind
increases with $m_{pts}$, therefore, for a fair comparison we use $m_{pts}=2$ in
other experiments.

We evaluated the impact of \pandora algorithm on HDBSCAN computation on \amdcpu
and \amdgpu. We used a  multithreaded \hdbscan implementation
\wangemst\cite{wang2021} as baseline. Additionally, we used GPU MST computation
from ArborX~\cite{prokopenko2022}  along with \pandora~ for computing dendrogram
for GPU implementation of \hdbscan. We based our results on two datasets:
\dataset{Hacc37M} and \dataset{Uniform100M3D}, primarily focusing on the
\(m_{pts}\) value, the sole parameter affecting these phases.

We show the results in \Cref{fig:hdbscan}. Overall combination of ArborX with
\pandora on \amdgpu is  8-12$\times$ faster than multithreaded \wangemst. And
dendrogram computation with \pandora on \amdgpu is 17-33$\times$ faster than
\unionfind in \wangemst. We also observe that \pandora uses less than a third of
the total \hdbscan time, while \unionfind can account for more than half.

With increasing $m_{pts}$ times for dendrogram computation
for both datasets. Going from \(m_{pts}=2\) to \(m_{pts}=16\), dendrogram
computation time in \pandora increased by 1.1-1.5$\times$. In contrast, for
\unionfind, the this time increased by a factor of 1.6-2.4$\times$

We also observe that the speed-up of dendrogram computation increases with
\(m_{pts}\). However, as \(m_{pts}\) rises, the EMST computation
demands more resources. Thus, the benefits of quicker dendrogram computation may
be counterbalanced by the more demanding \emst computation. 




\section{Conclusion}\label{sec:conclusion}

We presented a new parallel algorithm to construct a dendrogram using GPUs. We
showed described the details of the implementation. We gave the experimental
results confirming its performance portability and efficiency compared to the
existing approaches using a variety of datasets and hardware architectures. To
our knowledge, this is the first algorithm to construct a dendrogram on GPUs.
Our algorithm combined with \emst on GPUs allows us to produce \hdbscan
clustering in under a minute for datasets that fit into GPU memory.

In future work, we plan to improve the efficiency of individual kernels and
explore different algorithms for \mst compression—we believe there is still a
lot of untapped potential in this area. Another interesting possibility would be
to extend our approach so that it can process datasets too large for device
memory.

\unless\ifblind
\section*{Acknowledgment}
This research was supported by the Exascale Computing Project (17-SC-20-SC), a
collaborative effort of the U.S. Department of Energy Office of Science and
the National Nuclear Security Administration.

This research used resources of the Oak Ridge Leadership Computing Facility at
the Oak Ridge National Laboratory, which is supported by the Office of Science
of the U.S. Department of Energy under Contract No. DE-AC05-00OR22725.
\fi

\ifjournal
  \bibliographystyle{ACM-Reference-Format}
\else
  \bibliographystyle{apalike}
\fi
\bibliography{main}

\end{document}